\documentclass{article}

\usepackage[main,final]{neurips_2026}
\renewcommand{\acksection}{\section*{Acknowledgements}}

\usepackage[utf8]{inputenc}
\usepackage[T1]{fontenc}
\usepackage{amsmath,amssymb,amsfonts}
\usepackage{amsthm}
\usepackage{url}
\usepackage{booktabs}
\usepackage{graphicx}
\usepackage{algorithm}
\usepackage{algpseudocode}
\usepackage{nicefrac}
\usepackage{microtype}
\usepackage{xcolor}

\usepackage{dsfont}  
\usepackage{listings}
\usepackage[font=small]{caption}
\usepackage{wrapfig}
\usepackage{needspace}
\usepackage{hyperref}

\newtheorem{theorem}{Theorem}
\newtheorem{claim}{Claim}

\makeatletter
\renewcommand{\@makefnmark}{\hbox{$^{\@thefnmark}$}}
\makeatother

\makeatletter
\renewcommand{\@makefnmark}{\hbox{$^{\@thefnmark}$}}
\makeatother

\title{DualKV: Shared-Prompt Flash Attention for Efficient RL Training with Large Rollouts and Long Contexts}
\author{
  Jiading Gai\textsuperscript{1,*}\thanks{Work done while at AWS.} \quad
  Shuai Zhang\textsuperscript{2,*} \quad
  Xiang Song\textsuperscript{3}\footnotemark[2] \quad
  Bernie Wang\textsuperscript{4}\footnotemark[2] \quad
  George Karypis\textsuperscript{4,5}\footnotemark[2]
  \\[4pt]
  \textsuperscript{1}\textit{Siemens} \quad
  \textsuperscript{2}\textit{Amazon Web Services} \quad
  \textsuperscript{3}\textit{Google} \\
  \textsuperscript{4}\textit{NTT DATA AIVista} \quad
  \textsuperscript{5}\textit{University of Minnesota} \\
  \texttt{jiading.gai@siemens.com} \quad \texttt{shuaizs@amazon.com} \\
  \texttt{xiangsx@google.com} \quad \texttt{ywang02@mit.edu} \quad \texttt{karypis@umn.edu} \\[2pt]
  \textsuperscript{*}\textit{Equal contribution.} \\[2pt]
  \textcolor{blue}{\url{https://github.com/amazon-science/dualkv-flash-attn-for-rl}}
}

\begin{document}
\maketitle

\begin{abstract}
Modern RL post-training methods such as GRPO and DAPO train on $N$ response sequences of $R$ tokens sampled from a shared prompt of $P$ tokens, but standard FlashAttention replicates all $P$ prompt tokens $N$ times across both forward and backward passes --- duplicating compute and memory on identical hidden states.
In large-rollout, long-context RL training ($N{\geq}16$, $P{\geq}8\text{K}$), this redundancy dominates the policy update cost.
We observe that in decoder-only models, causal masking makes prompt representations invariant across sequences at every layer, so all per-token operations (norms, projections, MLP) and attention can process the prompt once --- a property not yet exploited at the kernel level for RL training.
We propose \textbf{DualKV}, the first FlashAttention kernel variant that eliminates shared-prompt replication during RL training, via (1)~fused CUDA forward and backward kernels that iterate over two disjoint KV regions --- shared context and per-sequence response --- in a single kernel launch, and (2)~a data-pipeline redesign in veRL that repacks $N(P{+}R)$ tokens into $P{+}NR$ tokens per micro-batch, extending the token reduction from attention to the entire model by a factor $\rho = N(P{+}R)/(P{+}NR)$.
DualKV is mathematically equivalent to standard attention and introduces no approximation.
On Qwen3-8B GRPO training with 8$\times$H100 GPUs ($N{=}32$, 8K-context), DualKV achieves $1.63$--$2.09\times$ policy-update speedup, enables $2\times$ larger micro-batches, and raises MFU from $36\%$ to $76\%$. Similar gains hold for DAPO ($2.47\times$ speedup, $77\%$ MFU). At 30B MoE scale on 16$\times$H100, DualKV achieves $3.82\times$ policy-update and $3.38\times$ end-to-end step speedup over FlashAttention (which requires 4-way Ulysses sequence parallelism to avoid OOM).
DualKV also extends to hybrid sliding/global attention with head dimension 512 (which FA2 does not
support) and integrates with Ulysses sequence parallelism, demonstrated on Gemma-4-31B GRPO at 64K context.
\end{abstract}

\section{Introduction}
\label{sec:intro}

Modern RL post-training methods such as GRPO~\citep{shao2024deepseekmath} and DAPO~\citep{yu2025dapo} compute log-probabilities and gradients for $N$ response sequences that share a prompt of $P$ tokens. With standard FlashAttention-2 (FA2)~\citep{dao2023flashattention2}, the training micro-batch contains $N$ sequences of length $S_i = P + R_i$, repeating the $P$-token prompt $N$ times. Each forward and backward pass recomputes $K, V$ activations and their gradients over $(N{-}1) \times P$ redundant prompt tokens per layer.

Recent work demonstrates that model accuracy continues to improve with increasingly large rollout factors. \citet{lightman2023lets} show that best-of-$N$ performance on the MATH dataset scales log-linearly up to $N{=}1860$. GRPO and related methods similarly benefit from large $N$ for more accurate reward signal estimation (e.g., DAPO uses $N{=}16$). However, the compute cost of the standard $N$-copy packing grows as $O(N \cdot S^2)$: each of the $N$ sequences independently recomputes attention over the full $P$ prompt tokens, duplicating $(N{-}1) \times P$ prompt tokens in both compute and memory at every layer. As we show in Section~\ref{sec:e2e-longreason}, eliminating this replication cuts policy-update time by $2\times$.

Prompt KV sharing for inference has been addressed from multiple angles: paged attention~\citep{kwon2023vllm} and prefix caching avoid redundant prompt KV storage via copy-on-write block tables, while bifurcated attention~\citep{athiwaratkun2024bifurcated} decomposes attention into shared-prompt and per-sequence phases for parallel decoding. However, these mechanisms target inference only, where shared prompt KV is read-only; they have no training backward, no autograd-compatible gradient accumulation for the shared prompt, and cannot be dropped into a RL policy-update pipeline. Training raises three new challenges: (1) the backward pass has $N$ sequences concurrently accumulating gradients into the shared KV buffer --- a pattern absent from inference that requires both atomic accumulation for race-free writes and fp32 accumulators with a final cast to match FA2's per-element precision; (2) autograd must aggregate gradients correctly across the context self-attention and decoded-attention calls that both touch the shared KV; and (3) the RL training framework's data pipeline must be restructured to group same-prompt responses per micro-batch (Appendix~\ref{app:pipeline-invariance}). DualKV addresses all three at the kernel and system level.

We observe that in decoder-only models, prompt hidden states are \emph{identical} across all $N$ sequences at every
layer: causal masking ensures each prompt token attends only to preceding prompt tokens, never to the differing response
tokens that follow, so the prompt representations are independent of which response is generated. Prefix
Grouper~\citep{liu2025prefixgrouper} applies this at the framework level but provides no kernel support --- it still
passes the full $N$-replicated KV through standard FA2, retaining the $O(N \cdot P \cdot d)$ memory bottleneck that
causes OOM and forces sequence parallelism. DualKV contributes the first FlashAttention kernel variant that eliminates
this replication, with fused forward and backward CUDA kernels that read shared KV from a single physical buffer and
accumulate gradients from $N$ concurrent sequences via fp32 atomic writes. By packing the micro-batch with a single
prompt copy and decomposing attention into context self-attention (computed once) and a fused DualKV kernel for decoded
attention, DualKV eliminates redundant prompt computation throughout the entire model --- not just attention, but also
norms, projections, and MLP layers --- yielding both memory and compute savings without approximation.
DualKV also supports head dimension 512 and hybrid sliding/global attention (Gemma-4, beyond FA2's
$d{\le}256$ limit) and composes with Ulysses sequence parallelism for ultra-long context; we validate both with
end-to-end GRPO on Gemma-4-31B at up to 64K, extending DualKV to the hybrid sliding/global attention adopted by
recent model architectures.

\section{Eliminating Prompt Replication in RL Training: The DualKV Design}
\label{sec:training}

Each GRPO training step processes $N$ response sequences of length $S_i = P + R_i$ that share a prompt of $P$ tokens. The step runs three passes before the optimizer step: (1)~\textbf{old\_log\_prob}: forward through $\pi_{\theta_{\text{old}}}$ for the importance ratio; (2)~\textbf{ref\_log\_prob}: forward through $\pi_{\text{ref}}$ for the KL penalty; and (3)~\textbf{policy update}: forward+backward through $\pi_\theta$ producing the policy gradient. All three passes see the same packed micro-batch with $(N{-}1) \cdot P$ redundant prompt tokens, so DualKV accelerates all three. The same applies to any RL method that packs $N$ same-prompt responses, e.g., DAPO.

\subsection{Baseline: Standard Packing with Replicated Prompts}

Standard implementations (e.g., veRL with FA2) pack the $N$ sequences into a single \texttt{flash\_attn\_varlen\_func} call using cumulative sequence lengths:
\begin{equation}
  \underbrace{[P, R_1]}_{\text{seq 1}},\; \underbrace{[P, R_2]}_{\text{seq 2}},\; \ldots,\; \underbrace{[P, R_N]}_{\text{seq } N}
  \quad\text{--- total } T_{\text{std}} = N(P + R) \text{ tokens}
  \label{eq:std-packing}
\end{equation}
The prompt tokens are \emph{replicated} $N$ times: each sequence carries its own copy of the prompt's hidden states, QKV projections, and attention computation. Every per-token operation (Norm, QKV projection, RoPE, MLP, output projection) processes all $N(P+R)$ tokens. For attention, FA2 computes $N$ independent causal self-attention passes, one per sequence $i$:
\begin{equation}
  O^{(i)} \;=\; \text{softmax}\!\left(\frac{Q^{(i)} (K^{(i)})^\top}{\sqrt{d}} + M_{\text{causal}}\right) V^{(i)}, \qquad i = 1, \ldots, N,
  \label{eq:baseline-fa2}
\end{equation}
where $Q^{(i)}, K^{(i)}, V^{(i)} \in \mathbb{R}^{(P+R_i) \times H \times d}$ are sequence $i$'s packed projections and $M_{\text{causal}}$ is the per-sequence causal mask. FA2 computes Eq.~\ref{eq:baseline-fa2} via tile-wise online softmax. FLOPs: $O(N \cdot S_i^2 \cdot H \cdot d)$.

\subsection{DualKV: Single-Prompt Packing}
\label{sec:dualkv-packing}

\textbf{Prompt invariance and redundancy.} In Eq.~\ref{eq:baseline-fa2}, split the output rows of $O^{(i)}$ into prompt-query rows $O_P^{(i)}$ (the first $P$) and response-query rows $O_{R_i}^{(i)}$ (the remaining $R_i$). For the prompt-query rows, causal masking blocks attention to response keys, so only the prompt keys participate:
\begin{equation}
  O_P \;=\; \text{softmax}\!\left(\frac{Q_P (K_P)^\top}{\sqrt{d}} + M_{\text{causal},P}\right) V_P,
  \label{eq:prompt-subblock}
\end{equation}
where $Q_P, K_P, V_P$ are the prompt-row projections. Because all $N$ sequences share the same prompt tokens, $Q_P, K_P, V_P$ are identical across $i$ (Appendix~\ref{app:prompt-invariant}), so $O_P$ is the same for every sequence. The $(N-1)$ prompt self-attention computations in the baseline are therefore redundant.

For the response-query rows, each response attends causally to \emph{both} the shared prompt keys and its own response keys:
\begin{equation}
  O_{R_i}^{(i)} \;=\; \text{softmax}\!\left(\frac{Q_{R_i}^{(i)} [K_P;\, K_{R_i}^{(i)}]^\top}{\sqrt{d}} + M_{\text{causal}}\right) [V_P;\, V_{R_i}^{(i)}],
  \label{eq:response-subblock}
\end{equation}
where $[\cdot\,;\,\cdot]$ denotes concatenation along the token axis. This block is not redundant --- each response has distinct queries --- but it reuses the same $K_P, V_P$ for all $i$.

\textbf{DualKV packing.} DualKV packs the micro-batch as a single shared prompt followed by $N$ per-sequence responses:
\begin{equation}
  \underbrace{[P]}_{\text{prompt (once)}},\; \underbrace{[R_1]}_{\text{resp 1}},\; \ldots,\; \underbrace{[R_N]}_{\text{resp } N}
  \quad\text{--- total } T_{\text{dk}} = P + NR \text{ tokens}
  \label{eq:dualkv-packing}
\end{equation}
All per-token operations (norms, projections, RoPE, MLP, output projection) now process $P + NR$ tokens instead of $N(P+R)$ --- saving $(N{-}1)P$ tokens throughout the entire model.

\textbf{Two-call decomposition.} Eq.~\ref{eq:prompt-subblock} and Eq.~\ref{eq:response-subblock} can be computed separately to eliminate the redundancy while preserving exact attention:
\begin{itemize}
  \item \textbf{Call 1 --- context self-attention} (\texttt{flash\_attn\_varlen\_func}, standard FA2): computes Eq.~\ref{eq:prompt-subblock} over a single copy of the prompt, \emph{once}. FLOPs: $O(P^2 \cdot H \cdot d)$.
  \item \textbf{Call 2 --- decoded attention} (\texttt{flash\_attn\_dualkv\_varlen\_func}, DualKV kernel): computes Eq.~\ref{eq:response-subblock} for all $N$ responses, attending to the shared $K_P, V_P$ (from Call~1) plus per-sequence $K_{R_i}, V_{R_i}$. FLOPs: $O(N \cdot R \cdot S \cdot H \cdot d)$.
\end{itemize}

\textbf{Why a new kernel.} Eq.~\ref{eq:response-subblock}'s attention pattern --- response queries reading from the concatenation $[K_P; K_{R_i}]$ --- cannot be consumed by standard FA2 without materializing $N$ copies of $K_P, V_P$ (reintroducing the replication we just eliminated). The DualKV kernel iterates over the two physically disjoint KV regions ($K_P, V_P$ shared across the batch; $K_{R_i}, V_{R_i}$ per-sequence) within a single launch. The kernel's interface and algorithms are detailed in Section~\ref{sec:kernel}.

\textbf{Data pipeline.} This packing requires that all $N$ responses from the same prompt are co-located on the same GPU within the same micro-batch. veRL's rollout engine already produces prompt-contiguous sequences; DualKV preserves this ordering by skipping veRL's \texttt{balance\_batch} step (which reorders by sequence length) and disabling intra-epoch shuffle in the mini-batch iterator. The resulting pipeline preserves the exact mini-batch gradient estimator (Appendix~\ref{app:pipeline-invariance}). When the micro-batch contains multiple prompt groups, DualKV handles each group independently with its own context KV and per-group \texttt{cu\_seqlens}, supporting heterogeneous prompt lengths within a single forward pass.

\textbf{Gradient correctness.} The shared $K_P, V_P$ receive gradients from both calls --- autograd sums the two contributions, while the DualKV kernel internally accumulates the $N$ per-sequence terms within Call~2:
\begin{equation}
  \frac{\partial \mathcal{L}}{\partial K_P} \;=\; \left(\frac{\partial \mathcal{L}}{\partial K_P}\right)_{\!\text{Call 1}} \;+\; \sum_{i=1}^{N} \left(\frac{\partial \mathcal{L}}{\partial K_P}\right)^{(i)}_{\!\text{Call 2}}
  \label{eq:dkp-accumulation}
\end{equation}
(same structure for $V_P$). The first term is Call~1's prompt self-attention gradient; the $\sum_{i=1}^{N}$ collects Call~2's per-response contributions, creating an $N$-way concurrent write into shared $K_P, V_P$ --- a pattern with no analog in FA2's backward, resolved at the kernel level in Section~\ref{sec:kernel-backward}. The total is mathematically identical to the FA2 baseline (Appendix~\ref{sec:equiv-theorem}).

\textbf{When DualKV helps.} The token reduction ratio $\rho = N(P+R)/(P+NR)$ governs the per-token operation speedup (norms, projections, RoPE, MLP, output projection) throughout the entire model; attention has additional structural savings (Appendix~\ref{app:theoretical-analysis}). $\rho$ is largest in two regimes: \textbf{large rollout} ($N \geq 16$, common in GRPO and DAPO; at $N{=}32, P{=}16\text{K}$ the reduction reaches $7.2\times$) and \textbf{long prompts} ($P \geq 16\text{K}$, prevalent in agentic tasks~\citep{jimenez2024swebench} and repository-level code generation~\citep{liu2023repobench}; at $P{=}64\text{K}, N{=}16$ the reduction reaches $14.3\times$). In practice, $\rho$ is bounded by the per-GPU micro-batch size (Section~\ref{sec:e2e-longreason}). See Appendix~\ref{app:rho-scaling} for the full analysis.

\subsection{Composing DualKV with Sequence Parallelism}
\label{sec:dualkv-sp-method}

DualKV and sequence parallelism (SP) are complementary: DualKV removes shared-prompt replication,
while SP distributes sequence computation across ranks. We implement and validate this composition
with Ulysses SP. DualKV repacking precedes SP slicing, so ranks partition $P{+}NR$ tokens instead
of $N(P{+}R)$. Standard Ulysses all-to-all exchanges gather the sequence and scatter attention heads
before attention, then reverse this redistribution afterward. The two-region DualKV kernel operates
unchanged.

SP also requires synchronizing the last prompt token's hidden state, which predicts the first token
of every response in its prompt group. This state resides on one rank, while the responses may span
multiple ranks. An all-reduce makes the shared state available across the SP group for computing
first-token log-probabilities. The extra payload is one hidden vector per prompt group, independent
of context length. Implementation details appear in Appendix~\ref{app:dualkv-sp-method};
Section~\ref{sec:dualkv-sp} and Appendix~\ref{app:gemma4-sp-evaluation} provide experimental validation.

\section{New Computational Primitive: The DualKV Kernel}
\label{sec:kernel}

This section describes the DualKV CUDA kernel that executes Call~2 of the two-call decomposition (Section~\ref{sec:dualkv-packing}). Implemented as an extension of FA2's codebase, the kernel accepts five input tensors --- a decoded query $q$, the shared context $K_c, V_c$ (single copy, stored once), and per-sequence varlen-packed decoded $K_d, V_d$, plus a \texttt{context\_seqlen}${}=P$ parameter that shifts the causal mask so decoded queries attend to all $P$ context keys and their own preceding decoded keys. We write $K_c, V_c$ and $K_d, V_d$ (following FlashAttention's naming) for $K_P, V_P$ and $K_{R_i}, V_{R_i}$ in Section~\ref{sec:dualkv-packing}. Full Python signature in Appendix~\ref{app:kernel-interface}; pseudocode in Appendix~\ref{app:kernel-algorithms}; theoretical analysis in Appendix~\ref{app:theoretical-analysis}.

\textbf{Head dimensions and sliding-window support for hybrid architectures.} Recent open models interleave
two attention layer types that fall outside FA2's supported range. Gemma-4-31B, for example, stacks
full-attention layers at head dimension $512$ with sliding-window layers at head dimension $256$ --- and FA2
supports only $d \le 256$, so it
cannot serve the $d{=}512$ global layers at all. To bring shared-prompt deduplication
to such architectures, we add two DualKV kernel variants specifically for Gemma-4 support: a head-dim-$512$
full-attention path for the global layers, and kernel-native causal sliding-window attention for the $d{=}256$ local
layers (mathematically equivalent to windowed FA2, verified against it). This lets a single DualKV forward run the full
hybrid stack, dispatching per layer by head dimension and window. Implementation details --- tile sizes and the
shared-memory/register strategy that makes $d{=}512$ fit --- are in Appendix~\ref{app:hdim-swa}.

\subsection{Forward Pass}
\label{sec:kernel-forward}

The DualKV forward kernel extends FA2's tiled online-softmax algorithm (tile size $B_N$) to iterate over two physically disjoint KV regions within a single kernel launch. For each decoded query $q_r^{(i)}$ (position $r$ in sequence $i$), the kernel:(1) \textbf{Context phase}: Iterate over $\lceil P / B_N \rceil$ tiles of $k_{\text{ctx}}, v_{\text{ctx}}$ (tile-level slices of $K_c, V_c$; lowercase denotes per-tile data following FA2's convention), shared across all $N$ sequences in the grid. Accumulate running softmax statistics $(m, \ell, \mathbf{o})$ using the online-softmax algorithm; (2) \textbf{Decoded phase}: Continue iteration over $\lceil R_i / B_N \rceil$ tiles of sequence $i$'s own $k_{\text{dec}}^{(i)}, v_{\text{dec}}^{(i)}$. The running statistics carry over from the context phase, so the two-phase iteration produces the same output as if $[k_{\text{ctx}}; k_{\text{dec}}^{(i)}]$ were stored contiguously; (3) \textbf{Causal masking}: All context positions ($j < P$) are unmasked for every decoded query; decoded positions are masked causally using the $P$-shifted logical position $P + r$ (the \texttt{context\_seqlen} parameter).

The kernel decouples the physical block index from the logical key position: physical block $n$ has logical start $n \cdot B_N$ when $n < \lceil P / B_N \rceil$ (context) and $P + (n - \lceil P / B_N \rceil) \cdot B_N$ otherwise (decoded), so the causal mask operator sees the two physically disjoint regions as a single contiguous causal sequence.

\textbf{Per-tile out-of-bounds (OOB) masking.} Because the context and decoded regions are tiled to $B_N$ independently, up to two tiles within a single kernel instance can be partial --- the last context tile (when $P \bmod B_N \neq 0$) and the sequence's last decoded tile (when $R_i \bmod B_N \neq 0$) --- rather than FA2's single-tile boundary. The kernel bounds-checks every tile load and applies the per-tile OOB mask on every iteration, forgoing FA2's fast-path unmasked-steps loop. This cost is negligible compared to the $(N{-}1) \cdot \lceil P / B_N \rceil$ context-tile passes eliminated.

The kernel supports grouped-query attention (GQA) and variable decoded lengths (via \texttt{cu\_seqlens\_q}). Its output matches Eq.~\ref{eq:response-subblock} exactly; full pseudocode is in Algorithm~\ref{alg:dualkv-forward} (Appendix~\ref{app:kernel-algorithms}), and the full derivation is in Appendix~\ref{app:dualkv-forward-derivation}.

\subsection{Backward Pass with fp32 Atomic Accumulation}
\label{sec:kernel-backward}

The backward pass computes gradients for all five inputs ($dQ$, $dK_c$, $dV_c$, $dK_d$, $dV_d$) via the main DualKV backward kernel (Algorithm~\ref{alg:dualkv-backward}) and three auxiliary launches (described below). Two pieces distinguish it from FA2's backward: (1) \textbf{Two-region grid iteration} mirroring the forward (Section~\ref{sec:kernel-forward}): per-block pointer dispatch selects $K_c/V_c$ versus $K_d/V_d$, logical key positions feed the causal mask, and per-tile OOB masking handles two independent partial tiles. (2) \textbf{fp32 atomic accumulation} for the shared context gradients, resolving a concurrent-write race and the precision loss from accumulating $N$ partial gradients in bf16, both unique to DualKV's shared context buffer.

\textbf{Context gradient structure.} The loss gradient with respect to the shared context KV is a sum over all $N$ sequences in the batch:
\begin{equation}
  \frac{\partial L}{\partial K_c} = \sum_{i=1}^{N} \left(\frac{\partial L}{\partial K_c}\right)^{(i)}, \quad
  \frac{\partial L}{\partial V_c} = \sum_{i=1}^{N} \left(\frac{\partial L}{\partial V_c}\right)^{(i)},
  \label{eq:dualkv-ctx-grad}
\end{equation}
where $(\partial L / \partial K_c)^{(i)}$ is the contribution from sequence $i$'s decoded queries attending to the context (closed form in Appendix~\ref{app:dualkv-backward-derivation}). This additive structure drives the backward kernel design: the $N$ summands are computed by $N$ separate thread blocks, creating the concurrent-write problem Algorithm~\ref{alg:dualkv-backward} resolves.

\begin{algorithm}[tbp]
\small
\caption{DualKV backward main kernel (launched on grid $(n_{\text{ctx}} + n_{\text{dec}},\; N,\; H)$)}
\label{alg:dualkv-backward}
\begin{algorithmic}[1]
\Require $dO$, $O$, $Q$, saved \texttt{softmax\_lse}, $K_c, V_c, K_d, V_d$
\Require fp32 scratch $dK_{c,\text{acc}}, dV_{c,\text{acc}}$ (shared across batch; zero-initialized); $dQ_{\text{acc}}$ (per batch, fp32)
\Require $D = \text{rowsum}(dO \odot O)$ (from \texttt{dot\_do\_o}, reused from FA2)
\For{each $(n, b, h)$ \textbf{in parallel}} \Comment{grid: (K-col block, batch, head)}
    \State $n_{\text{ctx}} \gets \lceil P / B_N \rceil$;\quad \texttt{is\_context} $\gets (n < n_{\text{ctx}})$
    \If{\texttt{is\_context}}
        \State $(K_n, V_n) \gets (K_c, V_c)[n B_N : (n{+}1) B_N]$;\quad $j_{\text{base}} \gets n B_N$
    \Else
        \State $(K_n, V_n) \gets (K_d^{(b)}, V_d^{(b)})[(n{-}n_{\text{ctx}}) B_N : \ldots]$;\quad $j_{\text{base}} \gets P + (n{-}n_{\text{ctx}}) B_N$
    \EndIf
    \State $acc_{dK}, acc_{dV} \gets 0$ in fp32
    \For{each Q-row block $Q_m$ in causal support of $K_n$}
        \State Load $Q_m, dO_m, \texttt{softmax\_lse}[m], D_m$ from HBM to on-chip SRAM
        \State On chip, compute $P = \exp(Q_m K_n^\top / \sqrt{d} - \texttt{lse}_m)$ with causal mask \Comment{per-tile OOB mask if partial}
        \State $acc_{dV} \mathrel{+}= P^\top dO_m$
        \State $dP \gets dO_m V_n^\top$;\quad $dS \gets P \odot (dP - D_m)$
        \State $acc_{dK} \mathrel{+}= dS^\top Q_m / \sqrt{d}$
        \State \texttt{atomicAdd}($dQ_{\text{acc}}^{(b)}[m]$,\; $dS\, K_n / \sqrt{d}$) \Comment{FA2's existing $dQ$ pattern}
    \EndFor
    \If{\texttt{is\_context}} \Comment{\textbf{$N$ blocks (one per $b$) race on same address}}
        \State \texttt{atomicAdd}($dK_{c,\text{acc}}[n B_N]$,\; $acc_{dK}$)\phantomsection\label{line:dkc-atomic}
        \State \texttt{atomicAdd}($dV_{c,\text{acc}}[n B_N]$,\; $acc_{dV}$)
    \Else \Comment{Decoded block: exclusive per-$b$ region, no race}
        \State $dK_d^{(b)}[\ldots] \gets \text{bf16}(acc_{dK})$;\quad $dV_d^{(b)}[\ldots] \gets \text{bf16}(acc_{dV})$
    \EndIf
\EndFor
\end{algorithmic}
\end{algorithm}

\textbf{Concurrent-write problem (race on line~\ref{line:dkc-atomic}).} In a standard FA2 backward, the grid is $(n, b, h)$ and each sequence $i$ holds a private $K_c^{(i)}, V_c^{(i)}$, so the batch index maps each block to a separate $dK_c^{(i)}$ region --- each block can cast its fp32 accumulator directly to bf16 and write with no contention. DualKV collapses this batch dimension for the shared context: the single $dK_c$ and $dV_c$ buffer has no batch offset, yet $N$ thread blocks --- one per decoded sequence --- concurrently reach Algorithm~\ref{alg:dualkv-backward}'s context epilogue. A standard direct-write epilogue would produce a data race. Resolving the race via bf16 \texttt{atomicAdd} would compound $N{-}1$ bf16 rounding steps on top of the initial cast, turning the cross-sequence reduction itself into a half-precision accumulation.

\textbf{fp32 atomic accumulation.} We resolve this with a two-stage design. Stage 1 (Algorithm~\ref{alg:dualkv-backward}): context blocks \texttt{atomicAdd} into an fp32 scratch buffer; decoded blocks bypass the scratch and write bf16 directly (they own disjoint $dK_d/dV_d$ regions). Stage 2 (Algorithm~\ref{alg:dualkv-bwd-convert}): after all Stage 1 blocks complete, a separate kernel casts the fp32 scratch to bf16. This reproduces FA2's single-cast precision profile despite the $N$-way concurrent write: all $N$ context contributions are summed in fp32, and only the final value is cast to bf16.

\begin{algorithm}[tbp]
\small
\caption{\texttt{convert\_dkv\_context}: fp32-to-bf16 cast, launched on grid $(n_{\text{ctx}},\; 1,\; H)$ after Algorithm~\ref{alg:dualkv-backward} completes (batch dimension collapsed because context is shared)}
\label{alg:dualkv-bwd-convert}
\begin{algorithmic}[1]
\Require fp32 scratch $dK_{c,\text{acc}}, dV_{c,\text{acc}}$ (aggregated over $N$ sequences by Algorithm~\ref{alg:dualkv-backward})
\For{each $(n, h)$ \textbf{in parallel}, and each $(\text{token}, \text{dim})$ in tile}
    \State $dK_c[n B_N{+}\text{token}, \text{dim}] \gets \text{bf16}(dK_{c,\text{acc}}[n B_N{+}\text{token}, \text{dim}])$
    \State $dV_c[n B_N{+}\text{token}, \text{dim}] \gets \text{bf16}(dV_{c,\text{acc}}[n B_N{+}\text{token}, \text{dim}])$
\EndFor
\end{algorithmic}
\end{algorithm}
\vspace{-1em}

\textbf{Analogy to FA2's $dQ$ pipeline.} FA2 uses an fp32-then-cast pipeline for $dQ$ under K-block-parallel scheduling: the main backward kernel atomic-adds into $dQ_{\text{acc}}$ (Algorithm~\ref{alg:dualkv-backward}'s $dQ$ \texttt{atomicAdd} line), and a separate \texttt{convert\_dq} kernel casts $dQ_{\text{acc}}$ to bf16. DualKV replicates this pipeline for $dK_c, dV_c$ --- now indexed by \emph{batch element} (not K-block) --- with Algorithm~\ref{alg:dualkv-backward}'s context-epilogue atomicAdd and the separate \texttt{convert\_dkv\_context} kernel (Algorithm~\ref{alg:dualkv-bwd-convert}) casting to bf16. The fp32 accumulator keeps the context gradient within standard fp32 rounding of the FA2 baseline; formal equivalence is proven in Section~\ref{sec:equiv-theorem}.

\section{Experiments}
\label{sec:experiments}

We first benchmark the DualKV attention kernel in isolation (Section~\ref{sec:kernel-speedup}), then evaluate end-to-end DAPO and GRPO training at H100 scale on LongReason (Section~\ref{sec:e2e-longreason}) and on multi-node MoE (Section~\ref{sec:e2e-moe-multinode}). We also study the memory scaling with context length (Section~\ref{sec:memory-scaling}). Appendix~\ref{app:gsm8k-shortprompt} confirms DualKV benefits extend to the short-prompt regime (GSM8K, A100). Our primary baseline is FA2 in veRL's FSDP worker~\citep{sheng2024hybridflow}; we also include FA3 in Section~\ref{sec:e2e-longreason}.

\subsection{Kernel-Level Speedup: Forward, Backward, and Combined}
\label{sec:kernel-speedup}

Table~\ref{tab:benchmark-fwdbwd} reports separate forward-only, backward-only, and combined wall-clock times on a single NVIDIA A100 GPU with Qwen3-8B model dimensions ($H{=}32$, $H_k{=}8$, $d{=}128$, GQA $4{:}1$), fp16 with gradient computation enabled. The backward pass has different compute characteristics than the forward --- it computes gradients $\partial \mathcal{L}/\partial Q$, $\partial \mathcal{L}/\partial K$, $\partial \mathcal{L}/\partial V$ and involves the gradient accumulation for shared $K_c, V_c$ described in Section~\ref{sec:kernel-backward} (Appendix~\ref{app:dualkv-math}, Eqs.~\ref{eq:dkc-total}--\ref{eq:dvc-total}). DualKV gradients match FA2 to half-precision rounding (verified via \texttt{torch.allclose} with $\text{atol}{=}10^{-3}$, $\text{rtol}{=}10^{-3}$).

\vspace{-1em}
\begin{table}[ht]
\centering
\scriptsize
\caption{FA2 vs.\ DualKV wall-clock time (ms) and peak memory across varying rollout factor $N$ and prompt length $P$; response length $R{=}2048$ throughout. OOM = out of memory.}
\label{tab:benchmark-fwdbwd}
\begin{tabular}{@{}rr|rrr|rrr|rrr|r@{}}
\toprule
& & \multicolumn{3}{c|}{\textbf{FA2 (ms)}} & \multicolumn{3}{c|}{\textbf{DualKV (ms)}} & \multicolumn{3}{c|}{\textbf{Speedup}} & \textbf{Mem} \\
$N$ & $P$ & fwd & bwd & f+b & fwd & bwd & f+b & fwd & bwd & f+b & $\downarrow$ \\
\midrule
28 & 4K   & 43.6  & 128.6  & 172.2  & 27.1  & 77.1   & 104.2  & 1.61$\times$ & 1.67$\times$ & 1.65$\times$ & 63\% \\
28 & 16K  & 381.3 & 1059.6 & 1441.0 & 98.7  & 273.0  & 371.7  & 3.86$\times$ & 3.88$\times$ & 3.88$\times$ & 85\% \\
16 & 32K  & 771.5 & 2111.8 & 2883.4 & 139.5 & 386.2  & 525.8  & 5.53$\times$ & 5.47$\times$ & 5.48$\times$ & 86\% \\
16 & 64K  & \multicolumn{3}{c|}{OOM} & 363.3 & 994.5  & 1357.8 & \multicolumn{3}{c|}{$\infty$} & --- \\
\bottomrule
\end{tabular}
\end{table}

Combined fwd+bwd speedup ranges from $1.65\times$ ($P{=}4\text{K}$) to $5.48\times$ ($P{=}32\text{K}$), with up to $86\%$ peak memory reduction; at $P{>}32\text{K}$ standard FA2 OOMs entirely. Backward speedup tracks forward within $3\%$ as expected: the attention backward computes the same $QK^\top$ products as the forward. The backward-to-forward ratio (${\sim}2.7\times$) is stable across all configurations for both FA2 and DualKV, confirming that DualKV's two-region backward kernel adds no structural overhead beyond attention's intrinsic backward cost. Gradient accumulation for shared $K_c, V_c$ (Section~\ref{sec:kernel-backward}) adds no measurable overhead --- fp32 atomic writes and the follow-up bf16 cast are bandwidth-bound on context positions only and land inside the backward's GEMM timing. On H100, our open-source \href{https://github.com/JiadingGai/dualkv-flash-attn-for-rl/tree/fa3_devel}{DualKV--FA3 implementation} achieves $1.21$--$7.82\times$ combined attention forward/backward speedup and $36$--$81\%$ lower peak allocated memory than FA3 (Appendix~\ref{app:dualkv-fa3}).

\begin{wraptable}{r}{0.52\textwidth}
\vspace{-1em}
\centering
\scriptsize
\setlength{\tabcolsep}{3pt}
\caption{Single-layer fwd+bwd: DualKV vs.\ PG vs.\ FA2 (Qwen3-8B, H100). Full sweep in Appendix~\ref{app:prefix-grouper-comparison}.}
\label{tab:prefix-grouper-main}
\begin{tabular}{@{}rr|rrr|rr|rr@{}}
\toprule
& & \multicolumn{3}{c|}{\textbf{Time (ms)}} & \multicolumn{2}{c|}{\textbf{DK Speedup}} & \multicolumn{2}{c}{\textbf{Mem (GB)}} \\
$P$ & mb & FA2 & DualKV & PG & vs FA2 & vs PG & DK & PG \\
\midrule
5K  & 32 & 597  & \textbf{180} & 626  & 3.32$\times$ & 3.48$\times$ & \textbf{13.7} & 69.8 \\
8K  & 16 & 549  & \textbf{128} & 465  & 4.28$\times$ & 3.63$\times$ & \textbf{9.9}  & 50.7 \\
16K & 8  & 553  & \textbf{125} & 435  & 4.41$\times$ & 3.47$\times$ & \textbf{9.0}  & 45.9 \\
32K & 4  & 810  & \textbf{246} & 514  & 3.29$\times$ & 2.08$\times$ & \textbf{12.5} & 43.0 \\
65K & 4  & OOM  & \textbf{654} & OOM  & $\infty$     & $\infty$     & \textbf{22.7} & OOM \\
131K& 8  & OOM  & \textbf{2097}& OOM  & $\infty$     & $\infty$     & \textbf{44.5} & OOM \\
\bottomrule
\end{tabular}
\end{wraptable}

\textbf{Comparison with Prefix Grouper.}
Table~\ref{tab:prefix-grouper-main} benchmarks against Prefix Grouper (PG)~\citep{liu2025prefixgrouper}, a framework-level shared-prompt optimization for training (no kernel-level deduplication) --- FA2 still receives $N$-replicated KV on padded tensors, so projections, MLP, and norms process the full $N$-copy prompt. For configurations in Table~\ref{tab:prefix-grouper-main} where both methods run, DualKV is $2$--$4\times$ faster and uses $3$--$5\times$ less peak memory than PG. At high micro-batch ($P{=}5\text{K}$, mb=32), PG's padded representation makes it even slower than FA2 varlen. For configurations in Table~\ref{tab:prefix-grouper-main} with $P{>}32\text{K}$, both FA2 and PG are marked OOM while DualKV runs. PG only optimizes attention (${\sim}30\%$ of layer FLOPs), while DualKV eliminates redundancy across all operations via packed varlen representation.

\subsection{End-to-End GRPO and DAPO Training on LongReason}
\label{sec:e2e-longreason}

\begin{wrapfigure}{R}{0.4\textwidth}
\centering
\includegraphics[width=\linewidth]{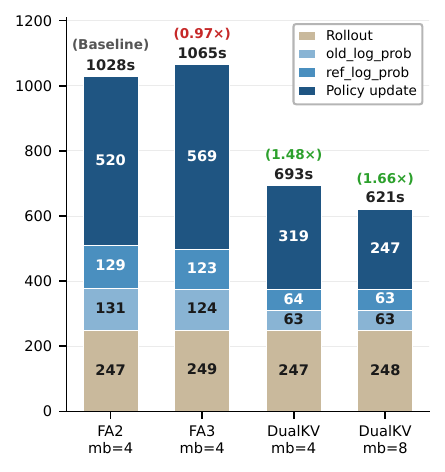}
\caption{Per-phase step breakdown.}
\label{fig:step-breakdown-main}
\vspace{-1em}
\end{wrapfigure}

To validate DualKV in a higher-$\rho$ regime with longer prompts and larger rollout factors, we evaluate on LongReason~\citep{ling2025longreason} --- a long-context mathematical reasoning dataset with prompts up to 8,192 tokens and response lengths up to 2,048 tokens --- using both GRPO and DAPO~\citep{yu2025dapo} (which removes the KL reference pass).

\textbf{Setup.} We train Qwen3-8B on a single \texttt{p5.48xlarge} instance (8$\times$H100-SXM5-80GB) with FSDP2 in BF16 and gradient checkpointing. Rollout generation uses vLLM with tensor parallelism $=2$ and $N{=}32$ responses per prompt. Training uses \texttt{train\_batch\_size}${}=128$ and \texttt{ppo\_mini\_batch\_size}${}=64$ (both fixed across all configurations), and we sweep \texttt{ppo\_micro\_batch\_size\_per\_gpu} $\in \{4, 8\}$. Each configuration runs for 25 training steps with validation every 2 steps. Although the rollout factor is $N{=}32$, the token reduction ratio $\rho$ (Eq.~\ref{eq:rho}) is bounded by the \emph{micro-batch} size --- only sequences co-located in the same forward/backward pass share a prompt. With $P{=}8192$, $R{=}2048$: $\rho = 2.5\times$ at mb=4 and $\rho = 3.3\times$ at mb=8, setting upper bounds on the measured policy-update speedup.

\textbf{Comparisons.} We compare four configurations (\texttt{mb} = per-GPU micro-batch size): FA2 (mb=4), the largest micro-batch FA2 can handle (mb=8 OOMs); FA3 (mb=4), evaluating the Hopper-specific kernel; DualKV (mb=4), an apple-to-apple comparison; and DualKV (mb=8), a configuration FA2 cannot run.

\textbf{Results.} Table~\ref{tab:e2e-longreason} and Figure~\ref{fig:longreason-comparison} summarize both experiments. All configs track nearly identical training accuracy curves (Figure~\ref{fig:longreason-comparison}A), confirming DualKV introduces no convergence degradation for either method (formal proof in Appendix~\ref{app:pipeline-invariance}).

\begin{figure}[!t]
\centering
\makebox[\textwidth][c]{\includegraphics[width=1.0\textwidth]{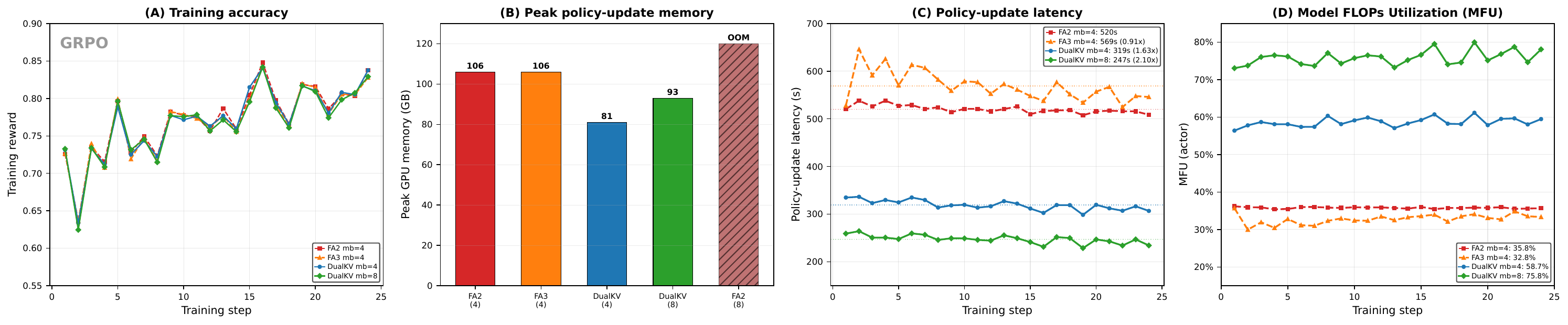}}
\makebox[\textwidth][c]{\includegraphics[width=1.0\textwidth]{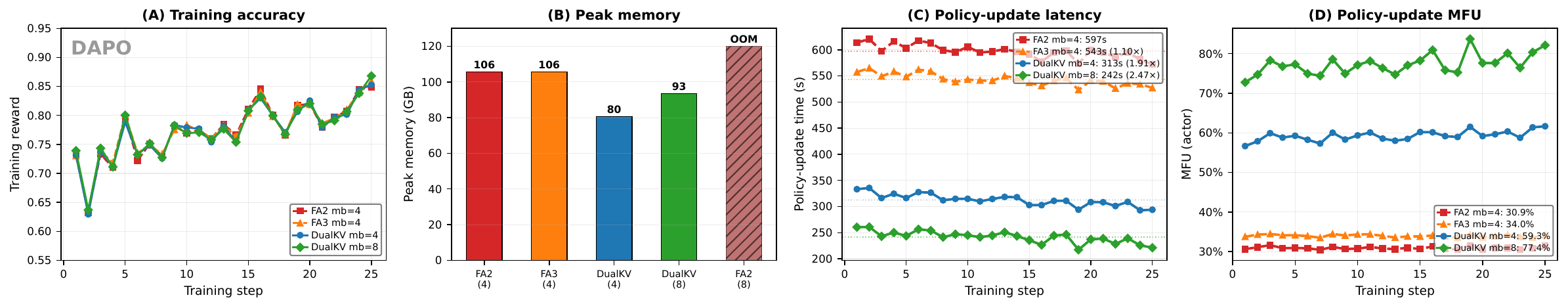}}
\caption{LongReason training (Qwen3-8B, 8$\times$H100), four configurations. \textbf{Top row}: GRPO ($N{=}32$, steps 1--24). \textbf{Bottom row}: DAPO ($N{=}32$, steps 1--25). Panels: \textbf{(A)}~Training reward --- all configs track nearly identical reward curves. \textbf{(B)}~Peak memory. \textbf{(C)}~Policy-update latency. \textbf{(D)}~Policy-update MFU. GRPO: DualKV mb=8 achieves $2.09\times$ speedup, $75.8\%$ MFU. DAPO: $2.47\times$ speedup, $77.4\%$ MFU.}
\label{fig:longreason-comparison}
\end{figure}

\begin{wraptable}{R}{0.6\textwidth}
\centering
\scriptsize
\setlength{\tabcolsep}{3pt}
\caption{End-to-end LongReason results. \texttt{Step} = total training-step time (s). \texttt{MFU} = policy-update Model FLOPs Utilization. Parentheses = speedup vs FA2 mb=4. Per-phase breakdown is in Figure~\ref{fig:step-breakdown-main}.}
\label{tab:e2e-longreason}
\begin{tabular}{@{}lc c c c c@{}}
\toprule
\textbf{Config} & \textbf{mb/GPU} & \textbf{Peak Mem} & \textbf{Step} & \textbf{MFU} & \textbf{Val Acc} \\
 & & \textbf{(GiB)}$^\dagger$ & \textbf{(s)} & \textbf{(\%)} & \textbf{(step 24)} \\
\midrule
FA2      & 4 & 106 & 1028 (1.00$\times$) & 35.8 & 0.777 \\
FA3      & 4 & 106 & 1065 (0.97$\times$) & 32.8 & 0.793 \\
DualKV   & 4 & 81                  &  693 (1.48$\times$) & 58.8 & 0.780 \\
DualKV   & 8 & 93                  &  621 (1.66$\times$) & 75.8 & 0.780 \\
FA2      & 8 & \multicolumn{4}{c}{OOM during \texttt{loss.backward()}} \\
\bottomrule
\end{tabular}
\vspace{2pt}

{\scriptsize $^\dagger$PyTorch allocator reservation; actual HBM use stays within 80\,GB.}
\end{wraptable}

\textbf{Speedups, memory, and MFU.} At mb=4, DualKV cuts policy-update time by $1.63\times$ (GRPO) and $1.91\times$ (DAPO), while reducing peak reserved memory from $106$\,GiB to $81$\,GiB. Doubling to mb=8 --- which FA2 OOMs on --- extends the speedup to $2.09\times$ (GRPO) and $2.47\times$ (DAPO), halving the gradient-accumulation passes per mini-batch ($2 \to 1$). Policy-update MFU rises from $36\%$/$31\%$ (FA2) to $76\%$/$77\%$ (DualKV mb=8) for GRPO/DAPO respectively --- a $2.1$--$2.5\times$ utilization gain. MFU uses veRL's FLOP estimate for the original batch, including repeated prompt computation that DualKV avoids. DAPO's higher speedups reflect the absence of the ref phase: the policy update dominates a larger fraction of step time, so DualKV's per-phase compression translates more directly to wall-clock gains.

\textbf{End-to-end step speedup.} Total training step time drops $1.66\times$/$1.82\times$ (GRPO/DAPO) at mb=8. The remaining gap versus the per-phase speedups is explained by rollout generation (${\sim}240$\,s/step, unaffected by DualKV), whose share of step time grows as training phases accelerate --- DualKV shifts the bottleneck from the policy update to rollout generation (Figure~\ref{fig:step-breakdown-main}). Full per-phase DAPO breakdowns are in Appendix~\ref{app:dapo-longreason}.

\subsection{Multi-Node RL Training with MoE Models}
\label{sec:e2e-moe-multinode}

\begin{wrapfigure}{R}{0.5\textwidth}
\vspace{-1.5em}
\centering
\includegraphics[width=\linewidth]{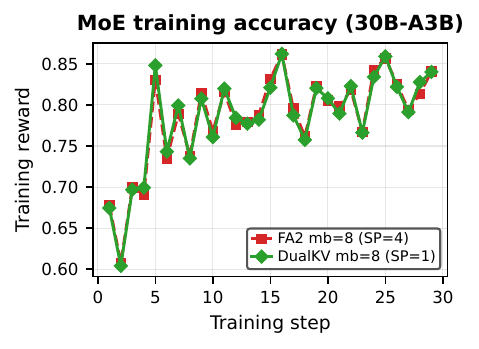}
\caption{Per-step training accuracy.}
\label{fig:moe-accuracy}
\end{wrapfigure}
To validate DualKV at production scale with cross-node communication and MoE expert routing, we train Qwen3-30B-A3B~\citep{yang2025qwen3} with GRPO on two \texttt{p5.48xlarge} nodes (16$\times$H100 GPUs, Elastic Fabric Adapter): FSDP2 + BF16 + gradient checkpointing; rollout via vLLM (TP${=}2$, $N{=}32$). All runs: \texttt{train\_batch\_size}${=}128$, \texttt{ppo\_mini\_batch\_size}${=}64$, \texttt{ppo\_micro\_batch\_size\_per\_gpu}${=}8$, $P_{\max}{=}8192$, $R_{\max}{=}2048$. Timing is mean over 29 steps. We compare three configurations: DualKV (SP=1), no sequence parallelism; FA2 (SP=2), which OOMs during the first policy update; and FA2 (SP=4), the minimum SP that avoids OOM.

\textbf{Results.} Table~\ref{tab:e2e-moe-multinode} summarizes the three configurations. DualKV delivers a $3.82\times$ policy-update speedup and $3.38\times$ total step speedup over FA2 SP=4, without requiring sequence parallelism. Both configs track nearly identical training accuracy curves (Figure~\ref{fig:moe-accuracy}), confirming DualKV introduces no convergence degradation at multi-node training scale (formal proof in Appendix~\ref{app:pipeline-invariance}).

\begin{table}[t]
\centering
\scriptsize
\setlength{\tabcolsep}{4pt}
\caption{Wall-clock values are means over 29 steps; parenthesized values are speedups relative to FA2 SP=4.}
\label{tab:e2e-moe-multinode}
\begin{tabular}{@{}lc c c c c c c c c@{}}
\toprule
\textbf{Config} & \textbf{mb/GPU} & \textbf{SP} & \textbf{Peak Mem} & \textbf{Policy Update} & \textbf{old\_log\_prob} & \textbf{ref\_log\_prob} & \textbf{Rollout} & \textbf{Step} & \textbf{Val Acc} \\
 & & & \textbf{(GB)} & \textbf{(s)} & \textbf{(s)} & \textbf{(s)} & \textbf{(s)} & \textbf{(s)} & \textbf{(step 28)} \\
\midrule
DualKV  & 8 & 1 & 103\rlap{$^\dagger$} & 1078 (3.82$\times$) & 129 (3.45$\times$) & 142 (3.58$\times$) & 215 (1.00$\times$) & 1564 (3.38$\times$) & 0.774 \\
FA2     & 8 & 4 &  92\rlap{$^\dagger$} & 4115 (1.00$\times$) & 447 (1.00$\times$) & 507 (1.00$\times$) & 214 (1.00$\times$) & 5284 (1.00$\times$) & 0.802 \\
FA2     & 8 & 2 & \multicolumn{7}{c}{OOM during first \texttt{loss.backward()}} \\
\bottomrule
\end{tabular}
\vspace{2pt}

{\scriptsize $^\dagger$PyTorch allocator reservation; actual HBM use stays within 80\,GB.}
\end{table}

\textbf{Speedups at multi-node MoE scale.}
DualKV delivers a $3.82\times$ policy-update speedup and $3.38\times$ total-step speedup (5284\,s $\to$ 1564\,s, throughput 293 $\to$ 991 tokens/s). The larger gains versus the 8B single-node range ($1.63$--$2.09\times$) arise because (i) MoE expert GEMMs amplify the per-token savings from DualKV's $(N{-}1) \cdot P$ reduction, and (ii) DualKV eliminates the all-to-all communication that FA2 with SP=4 pays on every attention call. Over 29 steps, DualKV completes in 12.6 hours vs.\ 42.6 hours for FA2 SP=4 on the same 16-GPU cluster (${\sim}\$6$K saved at AWS \texttt{p5.48xlarge} on-demand list price). Figure~\ref{fig:moe-multinode-breakdown} (Appendix~\ref{app:additional-figures}) visualizes the phase-wise breakdown.

\textbf{Eliminating sequence parallelism.} DualKV's memory reduction enables training at SP=1, whereas FA2 OOMs at SP=2 and requires SP=4. This translates memory savings into additional speedup: DualKV avoids the per-call all-to-all communication that SP imposes, which is why end-to-end gains ($3.38\times$) exceed what token dedup alone would predict.

\subsection{Memory Scaling with Context Length}
\label{sec:memory-scaling}

\begin{wrapfigure}
{R}{0.55\textwidth}
\vspace{-1.5em}
\centering
\includegraphics[width=\linewidth]{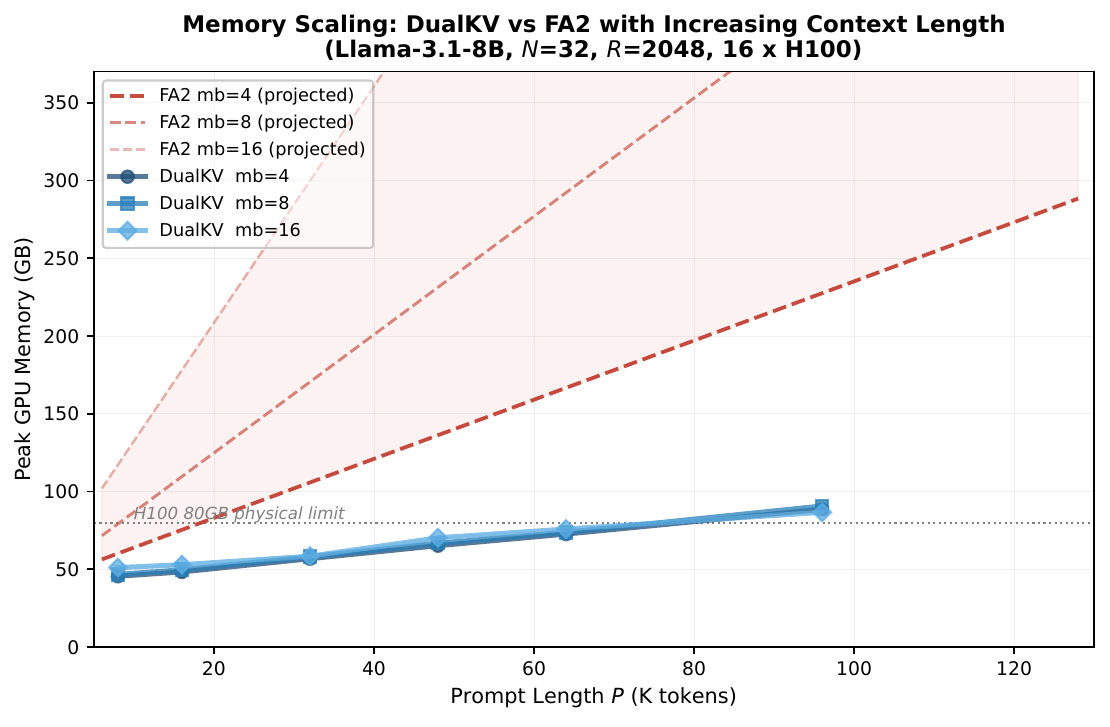}
\caption{Memory scaling: DualKV vs.\ FA2.}
\label{fig:memory-scaling-law-main}
\vspace{-1em}
\end{wrapfigure}

To characterize how DualKV's advantage scales with prompt length, we sweep $P \in \{8\text{K}\text{--}96\text{K}\}$ and $\text{mb} \in \{4, 8, 16\}$ on Llama-3.1-8B (16$\times$H100). This experiment also validates that DualKV generalizes across model families (Llama here versus Qwen3 in Sections~\ref{sec:e2e-longreason}--\ref{sec:e2e-moe-multinode}).

\textbf{Memory decouples from micro-batch size.} DualKV's memory follows $M_0 + c_P \cdot P + c_{\text{mb}R} \cdot \text{mb} \cdot R$: the prompt cost is paid once regardless of micro-batch size, so the three mb curves cluster tightly (Figure~\ref{fig:memory-scaling-law-main}). FA2's memory scales as $M_0 + c \cdot \text{mb} \cdot (P{+}R)$, diverging rapidly --- at $P{=}96\text{K}$, increasing mb from 4 to 16 adds only ${\sim}4$\,GB under DualKV ($<5\%$), whereas FA2 would require $225 \to 775$\,GB --- nearly $10\times$ the physical capacity of an H100. Practitioners can freely increase mb to maximize throughput without memory penalty in the long-prompt regime.

\subsection{DualKV with Sequence Parallelism on Gemma-4-31B}
\label{sec:dualkv-sp}

We extend DualKV to support Gemma-4-31B, demonstrating its applicability to modern hybrid attention
architectures. Our implementation adds support for head-dimension-$512$ global attention and
head-dimension-$256$ sliding-window attention. On eight H200 GPUs, DualKV completes three epochs of GRPO training on
LongReason with $N{=}16$ rollouts per prompt at context lengths from 16K to 64K. At 64K, eight-way
Ulysses sequence parallelism yields peak allocated memory of $78.9$\,GB per GPU. These experiments
demonstrate compatibility with hybrid attention and long-context distributed training. Full
configurations, phase timings, and SP correctness checks appear in
Appendix~\ref{app:gemma4-sp-evaluation}. Open-source code is available on the
\href{https://github.com/amazon-science/dualkv-flash-attn-for-rl/tree/gemma4-dev}{\texttt{gemma4-dev}}
branch.

\section{Related Work}

\textbf{Efficient attention for long sequences.} FlashAttention~\citep{dao2022flashattention,dao2023flashattention2,shah2024flashattention3} reduces attention memory from $O(S^2)$ to $O(S)$ via tiling and online softmax; Ring attention~\citep{liu2023ringattention}, Ulysses~\citep{jacobs2023ulysses}, Megatron-SP~\citep{korthikanti2023reducing}, and Context Parallelism~\citep{dubey2024llama3} distribute sequences across GPUs via sequence parallelism. All are orthogonal to DualKV's within-GPU shared-prompt deduplication and compose with it.

\textbf{Shared-prompt optimization.} Paged attention~\citep{kwon2023vllm}, RadixAttention~\citep{zheng2023sglang}, and bifurcated attention~\citep{athiwaratkun2024bifurcated} share prompt KV for inference only. \citet{wang2024accelerating} and FlashMask~\citep{wang2025flashmask} support shared-question DPO training. Prefix Grouper~\citep{liu2025prefixgrouper} is a framework-level data reorganization that applies the shared-prompt decomposition to RL policy updates but replicates $K, V$ $N$ times before the attention call, retaining $O(N \cdot P \cdot d)$ memory at the kernel level. DualKV contributes the kernel-level primitive that eliminates this replication: a FlashAttention variant that reads shared KV from a single physical buffer and accumulates gradients from $N$ concurrent sequences via fp32 atomic writes, fully realizing the decomposition's memory and compute savings for training (see Appendix~\ref{app:prefix-grouper-comparison} for the comparison).

\textbf{RL training systems.} veRL~\citep{sheng2024hybridflow}, OpenRLHF~\citep{hu2024openrlhf}, and TRL~\citep{vonwerra2020trl} implement PPO~\citep{schulman2017ppo}/GRPO~\citep{shao2024deepseekmath}/REINFORCE++~\citep{hu2025reinforceplusplus}/DAPO~\citep{yu2025dapo} policy updates by packing $N$ sequences with replicated prompts, without exploiting shared-prompt structure. DualKV is a drop-in replacement for the attention layer in these systems, requiring only data pipeline changes to group same-prompt responses.


\section{Conclusion}
\label{sec:conclusion}

DualKV eliminates the redundant prompt replication in RL policy updates by introducing a two-region FlashAttention
kernel that reads shared KV from a single physical buffer, paired with a data-pipeline redesign that extends the token
reduction to the entire model with provable gradient equivalence. Across dense and MoE models, DualKV achieves
$2$--$3.8\times$ policy-update speedup, raises MFU from $36\%$ to $76\%$, and greatly reduces the need for sequence
parallelism at long contexts. DualKV shifts the RL training bottleneck from the policy update to rollout generation. The
current design assumes a single shared prefix per micro-batch group and benefits scale with $P/R$. Generalizing the
kernel to arbitrary tree structures with shared sub-graphs (e.g., multi-turn conversations with partially overlapping
prefixes or tree-of-thought search) would lift the single-prefix constraint. Because DualKV and sequence
parallelism are orthogonal, DualKV's shared-prompt compression delays --- and in many cases eliminates --- the
need for SP; for large models at ultra-long context, we compose DualKV with Ulysses SP and
validate the combination on Gemma-4-31B at 64K context (Section~\ref{sec:dualkv-sp}), extending DualKV to modern
hybrid attention architectures. Porting to FA4/Blackwell remains a future direction.
Our implementation is open-sourced at
\textcolor{blue}{\url{https://github.com/amazon-science/dualkv-flash-attn-for-rl}}.

\begin{ack}
We thank Siemens, and in particular Bing Xiang, Fang Liu, Huzefa Rangwala, and Danielle C. Maddix, for helpful discussions and for supporting this research.
\end{ack}

\bibliographystyle{plainnat}
\bibliography{refs}

\appendix

\section{DualKV Formal Derivation}
\label{app:dualkv-math}

This appendix provides a self-contained mathematical derivation of the DualKV forward and backward passes and proves equivalence to the FA2 baseline.

\subsection{Setup and Notation}

Consider $N$ sequences sharing a prompt of $P$ tokens. Sequence $i$ has $R_i$ response tokens, with total length $S_i = P + R_i$. Let $T = P + \sum_{i=1}^N R_i$ denote the total packed token count.

In a standard transformer layer, let $h \in \mathbb{R}^{T \times D}$ denote the hidden states after the input norm. The QKV projection produces:
\begin{equation}
  Q = h W_Q, \quad K = h W_K, \quad V = h W_V
\end{equation}
where $W_Q \in \mathbb{R}^{D \times Hd}$, $W_K \in \mathbb{R}^{D \times H_k d}$, $W_V \in \mathbb{R}^{D \times H_k d}$, with $H$ query heads, $H_k$ KV heads, and head dimension $d$.

\subsection{Prompt Hidden State Invariant}
\label{app:prompt-invariant}

\begin{claim}[Prompt Hidden State Invariant]
In a decoder-only model, for any prompt position $j < P$, the hidden state $h_j$ is identical across all $N$ sequences at every layer.
\end{claim}

\begin{proof}
By induction on layers. At the embedding layer, prompt tokens are shared, so $h_j^{(i)} = h_j$ for all $i$. Suppose the invariant holds at the input to layer $\ell$. Every operation in the layer preserves it:
\begin{itemize}
  \item \textbf{RMSNorm}: per-token, so $\text{Norm}(h_j^{(i)}) = \text{Norm}(h_j)$ for all $i$.
  \item \textbf{QKV projection}: linear per-token, so $Q_j^{(i)} = Q_j$, $K_j^{(i)} = K_j$, $V_j^{(i)} = V_j$.
  \item \textbf{RoPE}: depends only on position index $j$, applied per-token.
  \item \textbf{Causal attention}: for $j < P$, the causal mask restricts attention to keys at positions $\{0, \ldots, j\}$, all prompt positions with identical $K$ and $V$. Combined with $Q_j^{(i)} = Q_j$ from the QKV-projection step, the attention output satisfies $O_j^{(i)} = O_j$ for all $i$.
  \item \textbf{Output projection, residual, MLP}: all per-token or element-wise.
\end{itemize}
The invariant holds at the output of layer $\ell$, completing the induction.
\end{proof}

\subsection{Forward Pass}
\label{app:dualkv-forward-derivation}

\paragraph{Packing.} We represent the micro-batch as a single packed sequence:
\begin{equation}
  \mathbf{x} = [\underbrace{p_0, \ldots, p_{P-1}}_{\text{prompt}},\; \underbrace{r_0^{(1)}, \ldots, r_{R_1-1}^{(1)}}_{\text{response 1}},\; \ldots,\; \underbrace{r_0^{(N)}, \ldots, r_{R_N-1}^{(N)}}_{\text{response } N}]
\end{equation}

All per-token operations process $T = P + \sum_i R_i$ tokens, saving $(N{-}1) \cdot P$ tokens versus the standard $\sum_i S_i = NP + \sum_i R_i$ layout.

\paragraph{QKV split.} After the QKV projection on $T$ tokens, we split at the prompt boundary $P$ (zero-copy tensor views):
\begin{align}
  Q_c &= Q_{[0:P]}, \quad K_c = K_{[0:P]}, \quad V_c = V_{[0:P]} \label{eq:ctx-split} \\
  Q_d &= Q_{[P:T]}, \quad K_d = K_{[P:T]}, \quad V_d = V_{[P:T]} \label{eq:dec-split}
\end{align}

\paragraph{Call 1: Context self-attention.} Standard causal attention over the $P$ prompt tokens, computed once:
\begin{equation}
  O_c = \text{CausalAttn}(Q_c, K_c, V_c) \in \mathbb{R}^{P \times H \times d}
  \label{eq:call1}
\end{equation}

\paragraph{Call 2: Decoded attention (DualKV kernel).} For sequence $i$ and decoded position $r \in \{0, \ldots, R_i - 1\}$, the query attends to the full key-value sequence $[K_c;\, K_d^{(i)}]$ with causal masking at logical position $P + r$:
\begin{equation}
  O_d^{(i)}[r] = \sum_{j=0}^{P + r} \alpha_j^{(i,r)} \cdot v_j
  \label{eq:call2}
\end{equation}
where $v_j$ indexes into $[V_c;\, V_d^{(i)}]$ and the attention weights are:
\begin{equation}
  \alpha_j^{(i,r)} = \frac{\exp\!\big(q_r^{(i)} \cdot k_j / \sqrt{d}\big)}{\sum_{j'=0}^{P+r} \exp\!\big(q_r^{(i)} \cdot k_{j'} / \sqrt{d}\big)}
\end{equation}

The \texttt{context\_seqlen}${}= P$ parameter shifts the kernel's internal position counter so that decoded query at kernel index $r$ maps to logical position $P + r$, ensuring:
\begin{itemize}
  \item All $P$ context positions are unmasked ($j < P \leq P + r$).
  \item Decoded positions are causally masked ($j \leq P + r$).
\end{itemize}

\paragraph{Reassembly.} The final attention output is:
\begin{equation}
  O = [O_c;\, O_d] \in \mathbb{R}^{T \times H \times d}
\end{equation}
which flows through $W_O$, residual addition, and MLP on $T$ tokens.

\subsection{Backward Pass}
\label{app:dualkv-backward-derivation}

Let $\mathcal{L}$ denote the training loss (computed only over response tokens). We receive $\partial \mathcal{L}/\partial O \in \mathbb{R}^{T \times H \times d}$ and split it:
\begin{equation}
  \frac{\partial \mathcal{L}}{\partial O_c} = \left(\frac{\partial \mathcal{L}}{\partial O}\right)_{\![0:P]}, \quad
  \frac{\partial \mathcal{L}}{\partial O_d} = \left(\frac{\partial \mathcal{L}}{\partial O}\right)_{\![P:T]}
\end{equation}

Note that $\partial \mathcal{L}/\partial O_c \neq 0$ in general: even though the loss is on response tokens only, gradients propagate back to prompt positions through subsequent layers' residual connections.

\paragraph{Call 2 backward (DualKV kernel).} Given $\partial \mathcal{L}/\partial O_d$ and the saved tensors from the forward pass, the backward produces:
\begin{align}
  \frac{\partial \mathcal{L}}{\partial Q_d},\quad
  \frac{\partial \mathcal{L}}{\partial K_d},\quad
  \frac{\partial \mathcal{L}}{\partial V_d}
\end{align}
and gradients with respect to the shared context KV:
\begin{align}
  \left(\frac{\partial \mathcal{L}}{\partial K_c}\right)_{\!d} &= \sum_{i=1}^{N} \sum_{r=0}^{R_i - 1} \sum_{h} \frac{\partial \mathcal{L}}{\partial O_d^{(i)}[r,h]} \cdot \frac{\partial O_d^{(i)}[r,h]}{\partial K_c} \label{eq:dkc-dec} \\
  \left(\frac{\partial \mathcal{L}}{\partial V_c}\right)_{\!d} &= \sum_{i=1}^{N} \sum_{r=0}^{R_i - 1} \sum_{h} \frac{\partial \mathcal{L}}{\partial O_d^{(i)}[r,h]} \cdot \frac{\partial O_d^{(i)}[r,h]}{\partial V_c} \label{eq:dvc-dec}
\end{align}

\paragraph{Implementation: fp32 gradient accumulation.}
The summation over $N$ sequences in Eq.~\ref{eq:dkc-dec}--\ref{eq:dvc-dec} introduces a concurrent-write pattern absent from standard FlashAttention: the single shared $dK_c$ and $dV_c$ buffer has no batch offset, yet $N$ thread blocks (one per decoded sequence) concurrently compute partial contributions to the same context positions. DualKV resolves this with fp32 atomic accumulation in a dedicated accumulator buffer, followed by a cast kernel that converts the accumulator to bf16. The kernel-level description and pseudocode are given in Section~\ref{sec:kernel-backward} (Algorithms~\ref{alg:dualkv-backward}--\ref{alg:dualkv-bwd-convert}). For the purposes of the mathematical derivation that follows, the relevant property is that the kernel produces $(\partial\mathcal{L}/\partial K_c)_d = \sum_i \sum_r \sum_h \partial\mathcal{L}/\partial O_d^{(i)}[r,h] \cdot \partial O_d^{(i)}[r,h]/\partial K_c$ exactly (up to fp32 rounding) as stated in Eq.~\ref{eq:dkc-dec}--\ref{eq:dvc-dec}.

\paragraph{Call 1 backward (standard FA2).} Given $\partial \mathcal{L}/\partial O_c$, produces:
\begin{align}
  \frac{\partial \mathcal{L}}{\partial Q_c},\quad
  \left(\frac{\partial \mathcal{L}}{\partial K_c}\right)_{\!c},\quad
  \left(\frac{\partial \mathcal{L}}{\partial V_c}\right)_{\!c}
\end{align}

\paragraph{Gradient accumulation.} Since $K_c$ and $V_c$ are the \emph{same tensor objects} passed to both Call~1 and Call~2, PyTorch autograd accumulates their gradients:
\begin{align}
  \frac{\partial \mathcal{L}}{\partial K_c} &= \left(\frac{\partial \mathcal{L}}{\partial K_c}\right)_{\!c} + \left(\frac{\partial \mathcal{L}}{\partial K_c}\right)_{\!d} \label{eq:dkc-total} \\
  \frac{\partial \mathcal{L}}{\partial V_c} &= \left(\frac{\partial \mathcal{L}}{\partial V_c}\right)_{\!c} + \left(\frac{\partial \mathcal{L}}{\partial V_c}\right)_{\!d} \label{eq:dvc-total}
\end{align}

The full QKV gradients are then reassembled:
\begin{equation}
  \frac{\partial \mathcal{L}}{\partial Q} = \left[\frac{\partial \mathcal{L}}{\partial Q_c};\, \frac{\partial \mathcal{L}}{\partial Q_d}\right], \quad
  \frac{\partial \mathcal{L}}{\partial K} = \left[\frac{\partial \mathcal{L}}{\partial K_c};\, \frac{\partial \mathcal{L}}{\partial K_d}\right], \quad
  \frac{\partial \mathcal{L}}{\partial V} = \left[\frac{\partial \mathcal{L}}{\partial V_c};\, \frac{\partial \mathcal{L}}{\partial V_d}\right]
\end{equation}

These propagate through the QKV projection backward to produce $\partial \mathcal{L}/\partial h$ and $\partial \mathcal{L}/\partial W_Q$, $\partial \mathcal{L}/\partial W_K$, $\partial \mathcal{L}/\partial W_V$.

\subsection{Equivalence to the FA2 Baseline}
\label{sec:equiv-theorem}

\begin{theorem}[Gradient Equivalence]
The parameter gradients produced by DualKV are identical to those produced by the FA2 baseline.
\end{theorem}

\begin{proof}
In the standard approach, each sequence $i \in \{1, \ldots, N\}$ has its own copy of the prompt KV: $K_c^{(i)}, V_c^{(i)}$. The gradient of $K_c^{(i)}$ has two components:
\begin{equation}
  \frac{\partial \mathcal{L}}{\partial K_c^{(i)}} = \underbrace{g_{\text{self}}^{(i)}}_{\text{prompt self-attn}} + \underbrace{g_{\text{cross}}^{(i)}}_{\text{decoded} \to \text{context}}
\end{equation}

The parameter gradient from the $W_K$ projection at prompt positions is:
\begin{equation}
  \frac{\partial \mathcal{L}}{\partial W_K}\bigg|_{\text{prompt}} = \sum_{i=1}^{N} h_c^\top \frac{\partial \mathcal{L}}{\partial K_c^{(i)}} = h_c^\top \left(\sum_{i=1}^{N} \frac{\partial \mathcal{L}}{\partial K_c^{(i)}}\right)
  \label{eq:dwk-standard}
\end{equation}
where $h_c$ is the prompt hidden state (identical across $i$ by the invariant, so it factors out of the sum).

In DualKV, the single $K_c$ receives gradient (Eq.~\ref{eq:dkc-total}):
\begin{equation}
  \frac{\partial \mathcal{L}}{\partial K_c} = \left(\frac{\partial \mathcal{L}}{\partial K_c}\right)_{\!c} + \sum_{i=1}^{N} g_{\text{cross}}^{(i)}
\end{equation}

The parameter gradient is:
\begin{equation}
  \frac{\partial \mathcal{L}}{\partial W_K}\bigg|_{\text{prompt}}^{\text{DualKV}} = h_c^\top \frac{\partial \mathcal{L}}{\partial K_c}
  \label{eq:dwk-dualkv}
\end{equation}

It remains to show that Eq.~\ref{eq:dwk-standard} equals Eq.~\ref{eq:dwk-dualkv}, i.e.:
\begin{equation}
  \sum_{i=1}^{N} \left(g_{\text{self}}^{(i)} + g_{\text{cross}}^{(i)}\right) = \left(\frac{\partial \mathcal{L}}{\partial K_c}\right)_{\!c} + \sum_{i=1}^{N} g_{\text{cross}}^{(i)}
  \label{eq:to-show}
\end{equation}

This reduces to showing:
\begin{equation}
  \sum_{i=1}^{N} g_{\text{self}}^{(i)} = \left(\frac{\partial \mathcal{L}}{\partial K_c}\right)_{\!c}
  \label{eq:self-equiv}
\end{equation}

In the standard approach, $g_{\text{self}}^{(i)}$ is the gradient of $K_c^{(i)}$ from the context self-attention of sequence $i$. This depends on $\partial \mathcal{L}/\partial O_c^{(i)}$, which propagates back from subsequent layers. While the forward values at prompt positions are identical across $i$, the backward gradients $\partial \mathcal{L}/\partial O_c^{(i)}$ generally differ (each sequence $i$ has different response tokens producing different loss contributions).

In DualKV, the single prompt copy receives the \emph{aggregated} upstream gradient:
\begin{equation}
  \frac{\partial \mathcal{L}}{\partial O_c} = \sum_{i=1}^{N} \frac{\partial \mathcal{L}}{\partial O_c^{(i)}}
\end{equation}

This aggregation happens naturally through the residual stream: in the packed layout, prompt hidden states at layer $\ell{+}1$ receive gradients from all $N$ response branches propagating backward through subsequent layers. Since all intervening operations (norms, projections, MLP, residual) are per-token and \emph{linear in the backward pass}, the gradient at the prompt positions is exactly the sum over all $N$ sequences' contributions.

The context self-attention backward is linear in $\partial \mathcal{L}/\partial O_c$:
\begin{equation}
  \left(\frac{\partial \mathcal{L}}{\partial K_c}\right)_{\!c} = f_{\text{attn-bwd}}\!\left(\frac{\partial \mathcal{L}}{\partial O_c}\right) = f_{\text{attn-bwd}}\!\left(\sum_{i=1}^{N} \frac{\partial \mathcal{L}}{\partial O_c^{(i)}}\right) = \sum_{i=1}^{N} f_{\text{attn-bwd}}\!\left(\frac{\partial \mathcal{L}}{\partial O_c^{(i)}}\right) = \sum_{i=1}^{N} g_{\text{self}}^{(i)}
\end{equation}

where $f_{\text{attn-bwd}}$ denotes the attention backward function mapping $dO$ to $dK$, which is linear in $dO$ (since the forward values $Q_c, K_c, V_c$ are fixed during backward). This establishes Eq.~\ref{eq:self-equiv} and completes the proof.
\end{proof}

\subsection{Complexity Summary}

This subsection summarizes the per-component FLOPs and memory scaling of FA2 baseline vs.\ DualKV in asymptotic notation, abstracting the concrete numbers of Section~\ref{sec:training}. The table shows that the savings are dominated by the token-count reduction $T_{\text{std}} \to T_{\text{dk}}$: every per-token cost (norms, projections, MLP, KV storage, activation memory) shrinks by the ratio $\rho = T_{\text{std}}/T_{\text{dk}}$, and only attention has additional structural savings beyond that.

\begin{table}[ht]
\centering
\caption{Complexity comparison. $T_{\text{std}} = NP + \sum_i R_i$, $T_{\text{dk}} = P + \sum_i R_i$.}
\label{tab:complexity}
\begin{tabular}{@{}lll@{}}
\toprule
\textbf{Component} & \textbf{FA2 baseline} & \textbf{DualKV} \\
\midrule
Norms, projections, MLP & $O(T_{\text{std}} \cdot D^2)$ & $O(T_{\text{dk}} \cdot D^2)$ \\
Attention FLOPs & $\sum_i O(S_i^2 \cdot H \cdot d)$ & $O(P^2 Hd) + \sum_i O(R_i S_i Hd)$ \\
KV memory per layer & $O(T_{\text{std}} \cdot H_k \cdot d)$ & $O(T_{\text{dk}} \cdot H_k \cdot d)$ \\
Activation memory & $O(T_{\text{std}} \cdot D)$ & $O(T_{\text{dk}} \cdot D)$ \\
\bottomrule
\end{tabular}
\end{table}

When $P \gg R$ and $R_i = R$ for all $i$:
\begin{equation}
  \frac{T_{\text{std}}}{T_{\text{dk}}} = \frac{NP + NR}{P + NR} \xrightarrow{P \gg R} N
\end{equation}

The savings approach a factor of $N$ (the batch size) across all components --- norms, projections, MLP, attention, and activation memory --- yielding near-linear scaling.

\section{Training-Step Gradient Equivalence}
\label{app:pipeline-invariance}

Appendix~\ref{app:dualkv-math} proves the DualKV kernel produces attention gradients identical to the FA2 baseline in exact arithmetic for any fixed packed input. This appendix lifts that kernel-level invariance to the training-step level. DualKV introduces two modifications to veRL's policy-update data pipeline, both motivated by the kernel's same-prompt-per-micro-batch packing contract:
\begin{enumerate}
  \item \textbf{Skip \texttt{balance\_batch}.} veRL's default pipeline reorders the rollout list by per-sample sequence length to equalize per-GPU token counts before dispatch. DualKV skips this step, since the seqlen permutation scatters rollouts from the same prompt across different GPUs, breaking the single-prompt packing the kernel requires. The global rollout multiset is unchanged; only its ordering and per-GPU partition differ.
  \item \textbf{Force \texttt{shuffle=False}} within the per-rank mini-batch iterator. Under the default pipeline, samples within each rank are reshuffled between PPO epochs. DualKV disables this to preserve contiguous prompt-grouped ordering across epochs.
\end{enumerate}

For one training step, let $B$ be the global multiset of $(p,r,A)$ triples produced by that step's rollout, where $p$ is a prompt, $r$ a sampled response, and $A=A(p,r)$ a scalar advantage. The policy-gradient estimator is $\hat g(B,\theta) = \sum_{(p,r,A) \in B} f(p,r,A,\theta)$, covering standard RL objectives (PPO, GRPO, DAPO). Let $\mathcal{P}_{\text{def}}$ denote veRL's default pipeline and $\mathcal{P}_{\text{dk}}$ the DualKV pipeline.

\begin{theorem}[No systematic bias at training-step aggregate]
For any parameters $\theta$,
\begin{equation}
  \sum_{(p,r,A) \in B} f^{\mathcal{P}_{\text{def}}}(p,r,A,\theta) \;=\; \sum_{(p,r,A) \in B} f^{\mathcal{P}_{\text{dk}}}(p,r,A,\theta) \qquad \text{in exact arithmetic.}
\end{equation}
\end{theorem}

\begin{proof}
The per-sample gradient $f(p,r,A,\theta)$ depends on $\theta$, the advantage $A$, and the per-token log-probabilities $\pi_\theta(r \mid p)$. Advantages are computed upstream of the policy update and are pipeline-independent. The kernel-level equivalence theorem of Appendix~\ref{app:dualkv-math} establishes the primary lemma we rely on: log-probabilities and their gradients at response positions are identical under $\mathcal{P}_{\text{def}}$ and $\mathcal{P}_{\text{dk}}$ for any fixed sample. Hence $f^{\mathcal{P}_{\text{def}}}(p,r,A,\theta) = f^{\mathcal{P}_{\text{dk}}}(p,r,A,\theta)$ termwise. The two pipeline modifications are both permutations of the rollout list, so the global multiset $B$ is the same under both pipelines. The equation follows by termwise equality and order-invariance of finite sums.
\end{proof}

\section{Theoretical Speedup and Memory Analysis}
\label{app:theoretical-analysis}

This appendix derives per-layer kernel-level memory savings and the closed-form attention speedup ratio for DualKV. These are theoretical ceilings; measured end-to-end numbers are in Sections~\ref{sec:e2e-longreason}--\ref{sec:e2e-moe-multinode} (and Appendix~\ref{app:gsm8k-shortprompt}).

\subsection{Memory Savings}
\label{app:theoretical-memory}

Per layer, the DualKV kernel avoids storing $N{-}1$ redundant copies of three tensors that an equivalent $N$-copy FA2 call would materialize:

\begin{table}[ht]
\centering
\small
\caption{Kernel-level memory savings per layer relative to an $N$-copy FA2 call over the same packed batch. Configuration: Qwen3-8B ($N{=}8$, $P{=}16384$, $R{=}2048$, $H{=}32$, $H_k{=}8$, $d{=}128$) in bf16.}
\label{tab:kernel-mem}
\begin{tabular}{@{}lll@{}}
\toprule
\textbf{Tensor saved} & \textbf{Per-layer savings (bytes)} & \textbf{Numerical} \\
\midrule
$K_c, V_c$ (shared prompt KV)         & $2 \cdot (N{-}1) \cdot P \cdot H_k \cdot d \cdot 2$ & 469\,MB \\
$Q$ at prompt positions (avoided vs.\ FA2 baseline) & $(N{-}1) \cdot P \cdot H \cdot d \cdot 2$ & 940\,MB \\
\texttt{softmax\_lse} for prompt positions & $(N{-}1) \cdot P \cdot H \cdot 4$              & 14\,MB \\
\midrule
\textbf{Total per layer}              &                                                      & \textbf{1423\,MB} \\
\bottomrule
\end{tabular}
\end{table}

These are \emph{kernel-level} savings: they count the per-layer attention-input tensors ($Q$ at prompt positions, $K_c$, $V_c$) and the prompt-position \texttt{softmax\_lse} that an FA2 baseline would materialize $N$ times and DualKV materializes once. DualKV's single-prompt packing also reduces framework-level activation memory across every per-token operation (norms, QKV/output projections, MLP, residual stream), since these now process $P + NR$ tokens instead of $N(P+R)$.

\subsection{Speedup Scaling}
\label{app:theoretical-speedup}

Let $W_{\mathrm{FA2}}$ and $W_{\mathrm{DualKV}}$ denote the approximate attention FLOP counts for FA2 and DualKV, respectively. Accounting for causal masking and canceling common FLOP factors, the ideal attention speedup is:
\begin{equation}
  \mathrm{Speedup}_{\mathrm{attn}}^{\mathrm{ideal}} = \frac{W_{\mathrm{FA2}}}{W_{\mathrm{DualKV}}} \approx \frac{N(P+R)^2/2}{P^2/2 + N(PR + R^2/2)} = \frac{N(P+R)^2}{P^2 + N(2PR + R^2)}.
  \label{eq:causal-attention-work}
\end{equation}

This ratio approaches $N$ as $P/R \to \infty$, with $N$ fixed. Across the three configurations in Table~\ref{tab:benchmark-fwdbwd} where FA2 completes, Eq.~\ref{eq:causal-attention-work} predicts ideal attention-work reductions of $1.75\times$, $4.20\times$, and $5.90\times$, compared with measured forward-and-backward speedups of $1.65\times$, $3.88\times$, and $5.48\times$, respectively. The measured speedups achieve approximately $92$--$94\%$ of these estimates, consistent with the predicted scaling over the tested range. Similar forward and backward speedups in Table~\ref{tab:benchmark-fwdbwd} indicate that shared-prompt reuse benefits both passes.

\subsection{Token Reduction Ratio \texorpdfstring{$\rho$}{ρ}: Scaling Analysis and Regimes}
\label{app:rho-scaling}

The token reduction ratio of DualKV over standard attention is:
\begin{equation}
  \rho = \frac{T_{\text{std}}}{T_{\text{dk}}} = \frac{N(P + R)}{P + NR}
  \label{eq:rho}
\end{equation}

This ratio governs the speedup across \emph{all} model components (norms, projections, MLP, activation memory), not just attention. We analyze its behavior along the two key axes: rollout factor $N$ and prompt-to-response ratio $P/R$.

\paragraph{Scaling with $N$ (rollout factor).} Fixing $P$ and $R$, taking the limit:
\begin{equation}
  \lim_{N \to \infty} \rho = \lim_{N \to \infty} \frac{N(P+R)}{P + NR} = \frac{P + R}{R} = 1 + \frac{P}{R}
\end{equation}

For large $N$, the speedup saturates at $1 + P/R$. With $P{=}16\text{K}$, $R{=}2\text{K}$: ceiling is $9\times$. Increasing $N$ beyond this point yields diminishing returns --- the bottleneck shifts to decoded-token compute $O(NR)$ which scales linearly in both approaches.

\paragraph{Scaling with $P/R$ (prompt length ratio).} Fixing $N$ and $R$, as $P \to \infty$:
\begin{equation}
  \lim_{P \to \infty} \rho = \lim_{P \to \infty} \frac{NP + NR}{P + NR} = N
\end{equation}

For long prompts, the speedup approaches $N$ --- near-linear in the rollout factor. This is the regime where DualKV is most impactful: the $(N{-}1)P$ redundant prompt tokens dominate total compute.

\paragraph{Concrete scenarios.} Table~\ref{tab:scenarios} shows the token reduction ratio $\rho$ for representative RL training configurations.

\begin{table}[ht]
\centering
\small
\caption{Token reduction ratio $\rho = T_{\text{std}} / T_{\text{dk}}$ for various $(N, P, R)$ configurations.}
\label{tab:scenarios}
\begin{tabular}{@{}rrrrl@{}}
\toprule
$N$ & $P$ & $R$ & $\rho$ & \textbf{Regime} \\
\midrule
8  & 2K   & 2K  & 1.8$\times$ & Short prompt, moderate rollout \\
8  & 16K  & 2K  & 4.5$\times$ & Long prompt, moderate rollout \\
8  & 64K  & 512 & 7.6$\times$ & Very long prompt \\
16 & 16K  & 2K  & 6.0$\times$ & Long prompt, large rollout \\
32 & 4K   & 2K  & 2.8$\times$ & Short prompt, very large rollout \\
32 & 16K  & 2K  & 7.2$\times$ & Long prompt, very large rollout \\
16 & 32K  & 2K  & 8.5$\times$ & Very long prompt, large rollout \\
16 & 64K  & 512 & 14.3$\times$ & Ultra-long prompt, large rollout \\
\bottomrule
\end{tabular}
\end{table}

Two regimes emerge where DualKV provides substantial benefits:
\begin{enumerate}
  \item \textbf{Large rollout factor} ($N \geq 16$): Common in GRPO and DAPO, where generating many candidate responses per prompt improves reward signal quality. At $N{=}32$ with $P{=}16\text{K}$, the reduction reaches $7.2\times$.
  \item \textbf{Long prompts} ($P \geq 16\text{K}$): Prevalent in agentic tasks, repository-level code generation, and multi-turn dialogue. At $P{=}64\text{K}$, even moderate rollout ($N{=}8$) yields $7.6\times$ reduction, and $N{=}16$ reaches $14.3\times$.
\end{enumerate}

\section{Software Environment}
\label{app:software-stack}

All kernel-level benchmarks (Section~\ref{sec:kernel-speedup}) and end-to-end GRPO training runs (Sections~\ref{sec:e2e-longreason}--\ref{sec:e2e-moe-multinode} and Appendix~\ref{app:gsm8k-shortprompt}) use an identical Python virtual environment. Configuration flags (\texttt{use\_dualkv}, \texttt{ulysses\_sequence\_parallel\_size}, \texttt{attn\_implementation}) are the only variables that differ between runs.

\begin{table}[ht]
\centering
\small
\caption{Software stack used across all experiments. ``DualKV build'' indicates our fork of \texttt{flash\_attn} 2.8.4 with the DualKV kernel (Section~\ref{sec:kernel}) added; base FA2 behavior is unchanged. ``DualKV integration'' indicates our patch to \texttt{verl} 0.7.0 that adds the data-pipeline changes (Section~\ref{sec:dualkv-packing}) and actor/trainer hooks needed to use the DualKV kernel.}
\label{tab:software-stack}
\begin{tabular}{@{}ll@{}}
\toprule
\textbf{Package} & \textbf{Version} \\
\midrule
Python               & 3.12 \\
\texttt{torch}       & 2.9.0+cu128 \\
\texttt{vllm}        & 0.12.0 \\
\texttt{flash\_attn} & 2.8.4 (DualKV build) \\
\texttt{verl}        & 0.7.0 (DualKV integration) \\
\texttt{transformers}& 4.57.6 \\
\texttt{ray}         & 2.55.1 \\
NCCL                 & 2.27.5 \\
\bottomrule
\end{tabular}
\end{table}

The veRL patch (${\sim}1200$ lines across 10 files) modifies the actor's forward pass, the trainer's data pipeline (skipping \texttt{balance\_batch}, enforcing prompt-grouping), and a monkey-patch hook that installs the DualKV attention call on supported model classes. Both the \texttt{flash\_attn} and \texttt{verl} patches are available at \textcolor{blue}{\url{https://github.com/amazon-science/dualkv-flash-attn-for-rl}}.

\section{Additional Experiments}
\label{app:additional-training}

\subsection{GRPO Training on GSM8K (Short Prompts)}
\label{app:gsm8k-shortprompt}

This appendix presents our short-prompt validation experiment (referenced in Section~\ref{sec:experiments} introduction). The experiment tests DualKV with short prompts ($\rho = 3.3$--$5.2\times$), where measured policy-update speedups are modest but memory savings still extend the feasible training frontier.

We evaluate DualKV in end-to-end GRPO training using the veRL framework on GSM8K~\citep{cobbe2021gsm8k}, comparing against standard FA2 packing across a range of prompt lengths.

\paragraph{Setup.} We train Qwen3-8B~\citep{yang2025qwen3} with GRPO on an 8-node cluster of NVIDIA A100-SXM4-40GB GPUs (64 GPUs total). The model is sharded via FSDP across 32 data-parallel ranks (tensor parallelism $=2$), with gradient checkpointing enabled and optimizer states offloaded to CPU. Each training step consists of: (1)~rollout generation via vLLM ($N{=}8$ responses per prompt, $R_{\max}{=}256$), (2)~log-probability computation, and (3)~policy update with gradient accumulation (micro-batch size $=8$ per GPU). We use \texttt{ppo\_mini\_batch\_size}${}=128$ and \texttt{train\_batch\_size}${}=512$.

\paragraph{Dataset.} To control prompt length, we construct synthetic GSM8K variants by prepending few-shot examples to each prompt until the target token count is reached. We sweep prompt lengths $P \in \{1\text{K}, 1.5\text{K}, 2\text{K}, 2.5\text{K}, 3\text{K}\}$ tokens, all with $R{=}256$ and $N{=}8$. Each configuration runs 3 training steps; we report metrics from step~3 (after 2 warmup steps) to exclude one-time initialization costs.

\paragraph{Prompt grouping.} DualKV requires all $N$ responses from the same prompt to be co-located on the same GPU within each micro-batch. We modify veRL's data pipeline to (1)~skip the default \texttt{balance\_batch} reordering when DualKV is enabled, preserving prompt-group contiguity, and (2)~add runtime assertions verifying that every micro-batch contains only same-prompt rollouts. The micro-batch size ($=8$) equals $N$, ensuring each forward/backward pass processes exactly one prompt group. Appendix~\ref{app:pipeline-invariance} proves that these data-pipeline modifications preserve the exact mini-batch gradient estimator $\hat g(B, \theta)$ under data-parallel training, so they do not alter training dynamics at any parameter setting.

\paragraph{Results.} Table~\ref{tab:e2e-gsm8k} reports peak GPU memory (reserved) and wall-clock time for both the full training step and the policy update phase (where DualKV applies). Rollout generation uses vLLM identically in both configurations and is excluded from the policy update timing.

\begin{table}[ht]
\centering
\small
\caption{End-to-end GRPO training: FA2 vs.\ DualKV at micro-batch size 8, $N{=}8$, $R{=}256$. Qwen3-8B on 64$\times$A100-40GB. ``Update'' = policy update time. OOM = out of memory during policy update.}
\label{tab:e2e-gsm8k}
\begin{tabular}{@{}r|r|rr|rr|rr@{}}
\toprule
& & \multicolumn{2}{c|}{\textbf{Peak Memory (GB)}$^\dagger$} & \multicolumn{2}{c|}{\textbf{Step Time (s)}} & \multicolumn{2}{c}{\textbf{Update Time (s)}} \\
$P$ & $\rho$ & FA2 & DualKV & FA2 & DualKV & FA2 & DualKV \\
\midrule
1K   & 3.3$\times$ & 37.8 & \textbf{26.5} & 457 & \textbf{449} & 175 & \textbf{168} \\
1.5K & 4.0$\times$ & 44.4 & \textbf{27.3} & 499 & \textbf{486} & 177 & \textbf{169} \\
2K   & 4.5$\times$ & 46.8 & \textbf{28.1} & 570 & \textbf{528} & 203 & \textbf{171} \\
2.5K & 4.9$\times$ & OOM  & \textbf{29.0} & --- & 584 & --- & 173 \\
3K   & 5.2$\times$ & OOM  & \textbf{29.8} & --- & 618 & --- & 176 \\
\bottomrule
\end{tabular}
\vspace{2pt}

{\footnotesize $\rho = N(P{+}R)/(P{+}NR)$ is the token reduction ratio; $^\dagger$PyTorch allocator reservation; actual HBM use stays within 40\,GB.\par}
\end{table}

\textbf{Memory scaling.} FA2 peak memory grows rapidly with prompt length (37.8$\to$44.4$\to$46.8~GB), exceeding the A100's 40~GB physical memory at $P{=}1.5\text{K}$ and hitting OOM at $P{=}2.5\text{K}$. DualKV memory grows slowly (26.5$\to$29.8~GB), remaining well within the 40~GB budget at all tested contexts. The memory gap widens from 11.3~GB ($-30\%$) at $P{=}1\text{K}$ to 18.7~GB ($-40\%$) at $P{=}2\text{K}$. Beyond $P{=}2\text{K}$, FA2 cannot train at all at micro-batch size 8, while DualKV continues to scale.

\textbf{Policy update speedup.} The policy update phase (forward/backward passes with gradient accumulation) shows increasing speedup: $1.04\times$ at $P{=}1\text{K}$, $1.05\times$ at $P{=}1.5\text{K}$, and $1.19\times$ at $P{=}2\text{K}$. The update time for DualKV is nearly constant across prompt lengths ($168$--$176$~s), confirming that the shared prompt KV is processed only once regardless of prompt length.

\textbf{End-to-end step time.} Total step time improvement is more modest ($1.02$--$1.08\times$) because rollout generation (vLLM inference) and log-probability computation --- which are identical for both methods --- dominate. At $P{=}2\text{K}$, the policy update accounts for only 36\% of FA2's step time but receives $1.19\times$ speedup, yielding an overall $1.08\times$ improvement.

\textbf{Extending the training frontier.} The most impactful result is that DualKV \emph{enables training} at context lengths where FA2 cannot run. At $P{=}3\text{K}$ with $N{=}8$ and micro-batch size 8, FA2 requires ${\sim}8 \times (3072{+}256) = 26,624$ tokens per micro-batch forward pass, while DualKV requires only $3072 + 8 \times 256 = 5,120$ tokens --- a $5.2\times$ reduction in activation memory. This enables practitioners to train with longer prompts on existing hardware without reducing batch size or adding gradient checkpointing overhead.

\paragraph{Numerical correctness.} We verified that DualKV matches standard FA2 to within half-precision rounding in two independent tests: (1)~kernel-level forward pass comparison using \texttt{torch.allclose} ($\text{atol}{=}10^{-3}$, $\text{rtol}{=}10^{-3}$ in fp16) across all supported head dimensions (64, 96, 128, 192, 256), GQA configurations, and context/decoded lengths; and (2)~end-to-end inference with Qwen2-1.5B via vLLM, where DualKV produced \emph{identical token sequences} to the standard attention baseline across 4 sequences sharing a common prompt (greedy decoding, 128 tokens). These results confirm that the two-phase attention decomposition and gradient accumulation introduce no numerical divergence beyond the expected half-precision rounding.

\subsection{DAPO Training on LongReason}
\label{app:dapo-longreason}

DualKV benefits any sampling-based RL method that packs $N$ same-prompt responses into a micro-batch. To demonstrate this generality beyond GRPO, we repeat the LongReason experiment (Section~\ref{sec:e2e-longreason}) using DAPO~\citep{yu2025dapo} --- a sampling-based RL method that removes the KL penalty (no reference model forward pass) and uses dynamic clipping with token-level loss aggregation.

\paragraph{Setup.} Identical to Section~\ref{sec:e2e-longreason}: Qwen3-8B on a single \texttt{p5.48xlarge} instance (8$\times$H100), FSDP2 + BF16 + gradient checkpointing, vLLM rollout (TP${=}2$, $N{=}32$), \texttt{train\_batch\_size}${=}128$, \texttt{ppo\_mini\_batch\_size}${=}64$, $P_{\max}{=}8192$, $R_{\max}{=}2048$, 25 training steps. DAPO-specific: \texttt{clip\_ratio\_low}${=}0.2$, \texttt{clip\_ratio\_high}${=}0.28$, \texttt{clip\_ratio\_c}${=}10.0$, \texttt{loss\_agg\_mode}${=}$token-mean, no KL loss. Because DAPO does not use a reference model, the \texttt{ref\_log\_prob} phase is absent; the training step consists of rollout, \texttt{old\_log\_prob}, and the policy update only.

\paragraph{Comparisons.} Same four configurations as Section~\ref{sec:e2e-longreason}: FA2 (mb=4), FA3 (mb=4), DualKV (mb=4), and DualKV (mb=8).

\paragraph{Results.} Table~\ref{tab:dapo-longreason} summarizes the results. DualKV delivers $1.91\times$ (mb=4) and $2.47\times$ (mb=8) policy-update speedup over FA2, comparable to the GRPO speedups ($1.63\times$/$2.09\times$). The slightly higher speedups reflect the absence of the ref phase: the policy update's share of step time is larger in DAPO, so DualKV's per-phase compression translates more directly to end-to-end gains. Step speedup is $1.63\times$ (mb=4) and $1.82\times$ (mb=8). MFU rises from $30.9\%$ (FA2) to $77.4\%$ (DualKV mb=8) --- a $2.5\times$ utilization gain. Peak memory drops from $106$\,GB to $80$\,GB (mb=4) and $93$\,GB (mb=8), confirming the same memory-headroom story. All four configs track nearly identical training accuracy curves (final reward $\approx 0.85$--$0.87$), confirming no convergence degradation.

\begin{table}[ht]
\centering
\scriptsize
\setlength{\tabcolsep}{3pt}
\caption{DAPO training on LongReason (Qwen3-8B, 8$\times$H100). Same setup as Table~\ref{tab:e2e-longreason} but with DAPO instead of GRPO. No \texttt{ref\_log\_prob} phase (DAPO removes the KL penalty). Speedups relative to FA2 mb=4.}
\label{tab:dapo-longreason}
\begin{tabular}{@{}lc c c c c c c c c@{}}
\toprule
\textbf{Config} & \textbf{mb/GPU} & \textbf{Peak Mem} & \textbf{Policy Update} & \textbf{old\_log\_prob} & \textbf{ref\_log\_prob} & \textbf{Rollout$^{*}$} & \textbf{Step} & \textbf{MFU} & \textbf{Reward} \\
 & & \textbf{(GiB)}$^\dagger$ & \textbf{(s)} & \textbf{(s)} & \textbf{(s)} & \textbf{(s)} & \textbf{(s)} & \textbf{(\%)} & \textbf{(step 25)} \\
\midrule
FA2      & 4 & 106 & 597 (1.00$\times$) & 153 (1.00$\times$) & --- & 236 (1.00$\times$) & 987 (1.00$\times$) & 30.9 & 0.848 \\
FA3      & 4 & 106 & 543 (1.10$\times$) & 142 (1.08$\times$) & --- & 234 (1.01$\times$) & 919 (1.07$\times$) & 34.0 & 0.859 \\
DualKV   & 4 &  80 & 313 (1.91$\times$) &  63 (2.43$\times$) & --- & 231 (1.02$\times$) & 606 (1.63$\times$) & 59.3 & 0.853 \\
DualKV   & 8 &  93 & 242 (2.47$\times$) &  63 (2.43$\times$) & --- & 237 (1.00$\times$) & 542 (1.82$\times$) & 77.4 & 0.868 \\
\bottomrule
\end{tabular}
\vspace{2pt}

{\scriptsize $^{*}$\texttt{Rollout} (vLLM inference) is not targeted by DualKV; $^\dagger$PyTorch allocator reservation; actual HBM use stays within 80\,GB.\par}
\end{table}

\paragraph{Comparison with GRPO.} The per-phase speedups are consistent with Section~\ref{sec:e2e-longreason}: DualKV compresses \texttt{old\_log\_prob} at ${\sim}2.4\times$ (vs ${\sim}2.1\times$ in GRPO) and the policy update at $1.91$--$2.47\times$ (vs $1.63$--$2.09\times$). The key takeaway is that DualKV uniformly accelerates any sampling-based RL method that packs $N$ same-prompt responses.

\begin{figure}[ht]
\centering
\includegraphics[width=\textwidth]{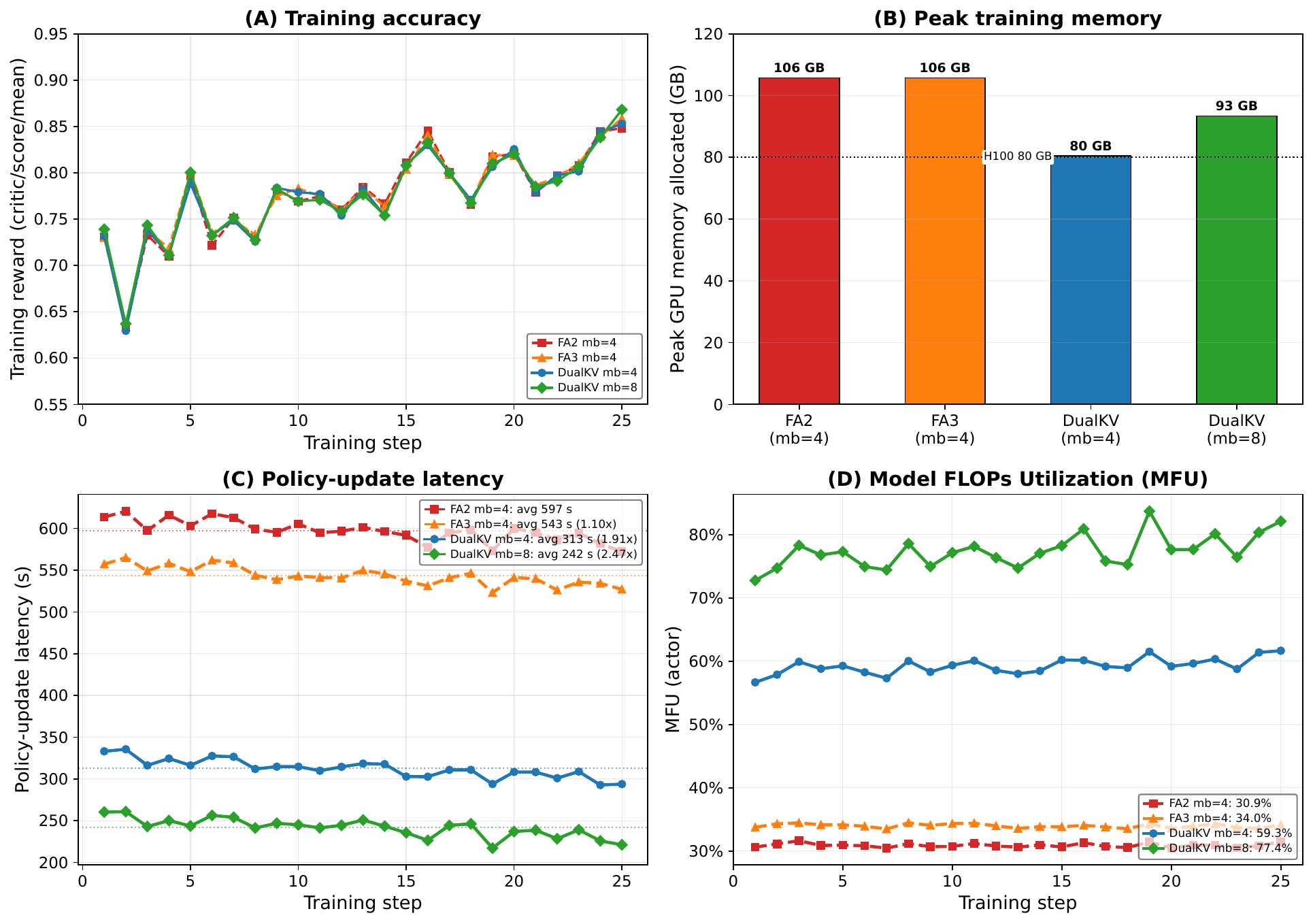}
\caption{DAPO training on LongReason (Qwen3-8B, 8$\times$H100). \textbf{(A)}~All four configurations track nearly identical reward curves, confirming no convergence degradation. \textbf{(B)}~Peak reserved memory: FA2/FA3 report $106$\,GiB, compared with $80$\,GiB for DualKV mb=4 and $93$\,GiB for mb=8. \textbf{(C)}~Policy-update latency: DualKV mb=8 achieves $2.47\times$ speedup over FA2. \textbf{(D)}~MFU rises from $31\%$ (FA2) to $77\%$ (DualKV mb=8).}
\label{fig:dapo-longreason-comparison}
\end{figure}

\subsection{Context-Length Scaling Sweep (Llama-3.1-8B)}
\label{app:context-sweep}

To validate that DualKV's benefits generalize beyond Qwen and to empirically characterize memory scaling, we conduct a context-length sweep on Llama-3.1-8B-Instruct~\citep{dubey2024llama3} across prompt lengths $P \in \{8\text{K}, 16\text{K}, 32\text{K}, 48\text{K}, 64\text{K}, 96\text{K}\}$ and micro-batch sizes $\text{mb} \in \{4, 8, 16\}$.

\paragraph{Setup.} Hardware: 2$\times$\texttt{p5.48xlarge} (16$\times$H100-80GB). GRPO with $N{=}32$, $R_{\max}{=}2048$, \texttt{train\_batch\_size}${=}128$, vLLM rollout (TP${=}2$), DualKV + FusedLinearForPPO enabled, \texttt{gpu\_memory\_utilization}${=}0.4$.

\paragraph{Memory scaling law.} Figure~\ref{fig:memory-scaling-law} plots peak GPU memory (\texttt{max\_memory\_allocated}) against prompt length for all 18 configurations. DualKV's memory follows:
\begin{equation}
  \text{mem}_{\text{DualKV}} = \underbrace{M_0}_{\substack{\text{model + optim +}\\\text{framework}}} + \underbrace{c_P \cdot P}_{\substack{\text{prompt activations}\\\text{(paid once)}}} + \underbrace{c_{\text{mb}R} \cdot \text{mb} \cdot R}_{\substack{\text{response activations}\\\text{(scales with mb)}}}
  \label{eq:mem-scaling}
\end{equation}
where the coefficients are estimated by least-squares regression across all 18 data points ($6$ prompt lengths $\times$ $3$ micro-batch sizes): $M_0 \approx 41.1$\,GB, $c_P \approx 0.475$\,GB/Ktoken, $c_{\text{mb}R} \approx 0.122$\,GB/Ktoken ($r^2 = 0.989$). The asymmetry $c_P \gg c_{\text{mb}R}$ arises because prompt tokens incur higher per-token activation cost under gradient checkpointing: block-boundary activations are retained for all $P$ positions across all layers, and the context self-attention backward (Call~1) saves \texttt{softmax\_lse} of size $O(P \times H)$. Response tokens in Call~2 attend only to their own shorter decoded subsequence, requiring less stored state. The key insight is structural: the prompt term $c_P \cdot P$ does not multiply by mb --- DualKV processes the shared prompt once regardless of how many responses are in the micro-batch. Under standard FA2 packing, each of the mb sequences carries its own full copy of the prompt, so memory scales as:
\begin{equation}
  \text{mem}_{\text{FA2}} = \underbrace{M_0}_{\text{same}} + \underbrace{c_P \cdot \text{mb} \cdot (P + R)}_{\text{prompt replicated mb times}}
  \label{eq:mem-scaling-fa2}
\end{equation}
making memory \emph{linear in mb}. FA2 values in Figure~\ref{fig:memory-scaling-law} are projections from Eq.~\ref{eq:mem-scaling-fa2} (using the same fitted $M_0$ and $c_P$) rather than measurements, because FA2 triggers CUDA OOM at these context lengths --- the standard $\text{mb} \times (P + R)$ packing exceeds the H100's 80\,GB HBM for most configurations in this sweep.

\paragraph{Practical implication.} When $P \gg \text{mb} \cdot R$ (the long-prompt regime common in agentic and reasoning tasks), memory is dominated by the prompt term and becomes effectively independent of micro-batch size. Practitioners can freely increase mb to maximize GPU utilization without memory penalty: at $P{=}96\text{K}$, increasing mb from 4 to 16 adds only ${\sim}4$\,GB ($<5\%$), whereas FA2 would require $225 \to 775$\,GB --- nearly $10\times$ the physical capacity of an H100.

\begin{figure}[ht]
\centering
\includegraphics[width=0.85\textwidth]{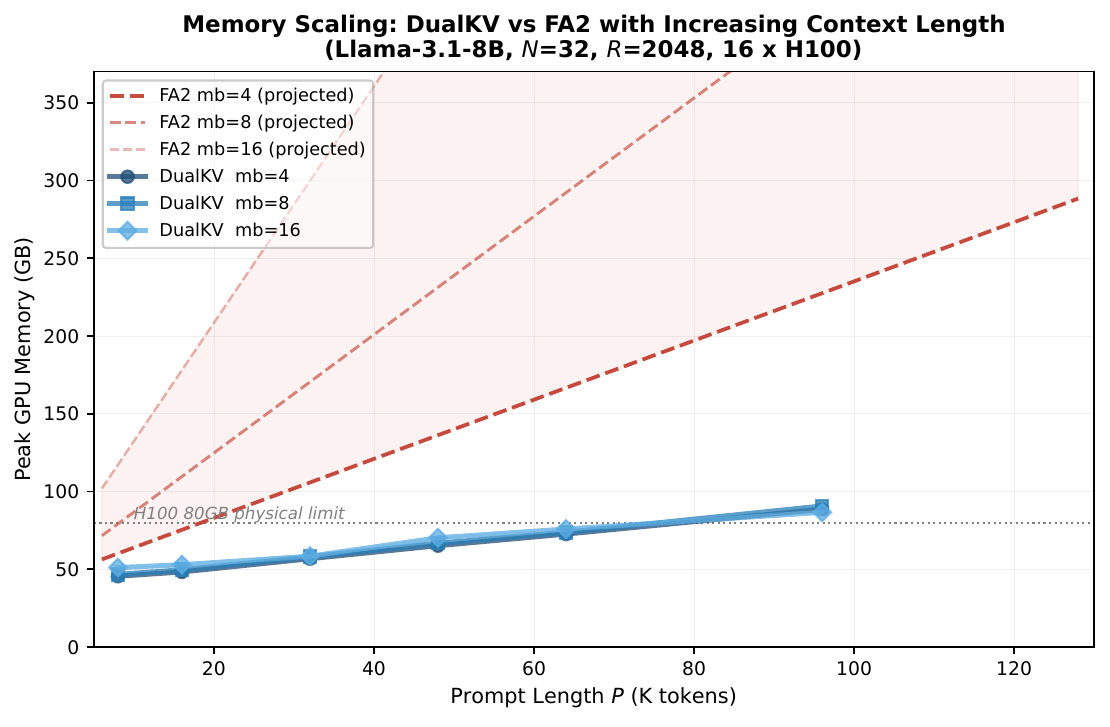}
\caption{Memory scaling law: DualKV (measured, solid) vs.\ FA2 (projected, dashed). Llama-3.1-8B, $N{=}32$, $R{=}2048$, 16$\times$H100. DualKV's three micro-batch curves cluster tightly ($\text{mem} \propto P + \text{mb} \cdot R$), while FA2's diverge ($\text{mem} \propto \text{mb} \cdot (P+R)$). FA2 curves are projected from the fitted model since FA2 OOMs at these configurations.}
\label{fig:memory-scaling-law}
\end{figure}

\subsection{Comparison with Prefix Grouper: Single-Layer Benchmark}
\label{app:prefix-grouper-comparison}

To empirically validate DualKV's advantage over framework-level prefix sharing, we benchmark against Prefix Grouper~\citep{liu2025prefixgrouper} on a single Qwen3-8B decoder layer (forward + backward) on one H100-80GB GPU. Prefix Grouper decomposes attention into prefix self-attention and suffix cross-attention (2$\times$ FA2 calls), but operates on \emph{padded} 2D tensors: all $N$ sequences are padded to the longest response length and the full padded tensor passes through QKV projections, output projection, LayerNorm, and MLP. We give Prefix Grouper the best possible backend (padded \texttt{flash\_attn\_func}, not the default HuggingFace wrapper) and compare against both standard FA2 varlen (the baseline) and DualKV.

\paragraph{Setup.} Qwen3-8B dimensions ($H{=}32$, $H_k{=}8$, $d{=}128$), bf16, response lengths $R \sim U(131, 2048)$. We sweep 54 configurations (nine prompt lengths $\times$ six micro-batch sizes); see the \href{https://github.com/JiadingGai/dualkv-flash-attn-for-rl/blob/b21eb94c470955a658056491613c125137cd2845/experiments/bench_long_context_sweep.py\#L40-L41}{benchmark script} for the complete settings. Timing: 2 warmup + 5 measured iterations. Correctness verified via \texttt{torch.allclose} against FA2 at smaller configs (both DualKV and PG pass; PG has zero error since it uses FA2 internally).

\begin{table}[ht]
\centering
\small
\caption{Selected single-layer forward+backward results: DualKV vs.\ Prefix Grouper (PG) vs.\ FA2. Time in ms, memory in GB. ``DK/FA2'' and ``DK/PG'' are DualKV's speedup ratios. OOM = out of memory on H100-80GB.}
\label{tab:prefix-grouper-comparison}
\setlength{\tabcolsep}{3.5pt}
\begin{tabular}{@{}rr|rrr|rr|rr@{}}
\toprule
& & \multicolumn{3}{c|}{\textbf{Time (ms)}} & \multicolumn{2}{c|}{\textbf{DualKV Speedup}} & \multicolumn{2}{c}{\textbf{Peak Memory (GB)}} \\
$P$ & mb & FA2 & DualKV & PG & vs FA2 & vs PG & DualKV & PG \\
\midrule
5K   & 4  & 93.2   & \textbf{36.6}  & 72.5   & 2.55$\times$ & 1.98$\times$ & \textbf{4.0}  & 9.0 \\
5K   & 8  & 153.3  & \textbf{61.4}  & 156.9  & 2.50$\times$ & 2.55$\times$ & \textbf{5.5}  & 18.1 \\
5K   & 16 & 338.9  & \textbf{110.4} & 314.2  & 3.07$\times$ & 2.85$\times$ & \textbf{9.0}  & 35.4 \\
5K   & 32 & 597.2  & \textbf{179.9} & 626.3  & 3.32$\times$ & 3.48$\times$ & \textbf{13.7} & 69.8 \\
\midrule
8K   & 4  & 148.9  & \textbf{57.2}  & 124.3  & 2.60$\times$ & 2.17$\times$ & \textbf{4.9}  & 12.8 \\
8K   & 8  & 245.4  & \textbf{70.3}  & 232.1  & 3.49$\times$ & 3.30$\times$ & \textbf{6.5}  & 25.8 \\
8K   & 16 & 549.2  & \textbf{128.2} & 465.4  & 4.28$\times$ & 3.63$\times$ & \textbf{9.9}  & 50.7 \\
8K   & 32 & \multicolumn{1}{c}{OOM} & \textbf{184.2} & \multicolumn{1}{c|}{OOM} & $\infty$ & $\infty$ & \textbf{14.6} & --- \\
\midrule
16K  & 4  & 329.9  & \textbf{108.9} & 245.7  & 3.03$\times$ & 2.26$\times$ & \textbf{7.5}  & 22.9 \\
16K  & 8  & 553.0  & \textbf{125.4} & 434.9  & 4.41$\times$ & 3.47$\times$ & \textbf{9.0}  & 45.9 \\
16K  & 16 & \multicolumn{1}{c}{OOM} & \textbf{205.8} & \multicolumn{1}{c|}{OOM} & $\infty$ & $\infty$ & \textbf{12.5} & --- \\
\midrule
32K  & 2  & 408.0  & \textbf{220.3} & 308.1  & 1.85$\times$ & 1.40$\times$ & \textbf{11.8} & 22.1 \\
32K  & 4  & 810.2  & \textbf{246.4} & 513.6  & 3.29$\times$ & 2.08$\times$ & \textbf{12.5} & 43.0 \\
32K  & 8  & \multicolumn{1}{c}{OOM} & \textbf{266.8} & \multicolumn{1}{c|}{OOM} & $\infty$ & $\infty$ & \textbf{14.1} & --- \\
\midrule
65K  & 2  & 1173.5 & \textbf{614.1} & 773.3  & 1.91$\times$ & 1.26$\times$ & \textbf{21.9} & 42.3 \\
65K  & 4  & \multicolumn{1}{c}{OOM} & \textbf{653.9} & \multicolumn{1}{c|}{OOM} & $\infty$ & $\infty$ & \textbf{22.7} & --- \\
65K  & 16 & \multicolumn{1}{c}{OOM} & \textbf{892.6} & \multicolumn{1}{c|}{OOM} & $\infty$ & $\infty$ & \textbf{27.7} & --- \\
\midrule
131K & 8  & \multicolumn{1}{c}{OOM} & \textbf{2096.7} & \multicolumn{1}{c|}{OOM} & $\infty$ & $\infty$ & \textbf{44.5} & --- \\
131K & 32 & \multicolumn{1}{c}{OOM} & \textbf{2883.0} & \multicolumn{1}{c|}{OOM} & $\infty$ & $\infty$ & \textbf{52.6} & --- \\
\bottomrule
\end{tabular}
\end{table}

\textbf{PG versus FA2 latency.} Prefix Grouper avoids redundant prompt self-attention but incurs padding, KV replication, and concatenation overhead. For responses of lengths $R_i$, it expands $P+\sum_i R_i$ distinct K/V token positions into $NP+\sum_i R_i$ positions for response attention, plus padding, adding latency and memory overhead. Across the reported configurations where both methods run, our padded PG-style baseline's latency ranges from 36.6\% below to 4.9\% above FA2's.

\textbf{DualKV speedup.} Across the reported configurations where both methods run, DualKV achieves $1.26$--$3.63\times$ speedup over PG and $1.85$--$4.41\times$ over FA2. The speedup comes from processing fewer tokens through \emph{every} operation in the layer. At the paper's primary config ($P{=}8\text{K}$, mb=8), DualKV processes 18K tokens vs.\ FA2's 75K and PG's 81K --- a $4.2\times$ reduction that translates directly to a $3.5\times$ wall-clock speedup over FA2.

\textbf{Memory gap is catastrophic for PG.} PG allocates a padded $[\text{mb}, P+R_{\max}, H]$ tensor, consuming $1.87$--$5.12\times$ the memory of DualKV's packed representation. At $P{=}8\text{K}$ mb=16, PG uses 50.7\,GB vs.\ DualKV's 9.9\,GB ($5.1\times$ reduction); at mb=32, PG OOMs while DualKV fits in 14.6\,GB.

\textbf{Scalability.} Of 54 total configurations tested (9 prompt lengths $\times$ 6 micro-batch sizes), DualKV runs 36 configs, while both FA2 and PG run only 21. Beyond $P{=}131\text{K}$, only DualKV can scale past mb=1 on a single H100.

\textbf{Root cause: PG only optimizes attention.} Attention accounts for ${\sim}30\%$ of a decoder layer's FLOPs; the remaining 70\% (QKV projections, output projection, MLP, norms) still operates on the full padded tensor with $N$ duplicated prompt copies. DualKV eliminates redundancy across \emph{all} operations by packing the prompt once in a varlen representation.

\clearpage
\section{DualKV Kernel Algorithms}
\label{app:kernel-algorithms}

This appendix gives full pseudocode for the DualKV forward and backward kernels (referenced from Sections~\ref{sec:kernel-forward}--\ref{sec:kernel-backward}). The listings elide K/V parallels: gradients, loads, and writes for $V$ mirror those for $K$ and are omitted for brevity.

\begin{algorithm}[H]
\small
\caption{DualKV forward kernel (one thread block, assigned $(m, b, h)$ = query-row block, batch, head)}
\label{alg:dualkv-forward}
\begin{algorithmic}[1]
\Require $Q_m$ for batch $b$, head $h$; shared $K_c, V_c \in \mathbb{R}^{P \times H_k \times d}$; per-sequence $K_d^{(b)}, V_d^{(b)} \in \mathbb{R}^{R_b \times H_k \times d}$; tile size $B_N$; $P$ (= \texttt{context\_seqlen})
\State $n_{\text{ctx}} \gets \lceil P / B_N \rceil$
\State $n_{\text{dec}} \gets \lceil R_b / B_N \rceil$
\State $n_{\max} \gets n_{\text{ctx}} + n_{\text{dec}}$
\State Apply causal pruning: reduce $n_{\max}$ so no tile exceeds the causal horizon of $Q_m$
\State Init $m_i \gets -\infty$,\; $\ell_i \gets 0$,\; $\mathbf{o}_i \gets 0$ \Comment{online-softmax state per row in $Q_m$}
\For{$n = n_{\max}{-}1$ down to $0$} \Comment{iterate K-tiles in reverse (FA2 convention)}
    \If{$n < n_{\text{ctx}}$} \Comment{context tile: shared buffer}
        \State $K_n, V_n \gets K_c[n B_N : n B_N + B_N], V_c[\ldots]$
        \State $j_{\text{base}} \gets n \cdot B_N$ \Comment{logical K position}
    \Else \Comment{decoded tile: per-sequence buffer}
        \State $K_n, V_n \gets K_d^{(b)}[(n{-}n_{\text{ctx}}) B_N : \ldots]$
        \State $j_{\text{base}} \gets P + (n - n_{\text{ctx}}) \cdot B_N$
    \EndIf
    \State Mask out-of-bounds lanes (partial tile at last context / last decoded tile)
    \State $S \gets Q_m K_n^\top / \sqrt{d}$
    \State Apply causal mask using $j_{\text{base}}$ (logical K position) vs.\ $Q_m$'s logical row position
    \State $m_i^{\text{new}} \gets \max(m_i,\; \text{rowmax}(S))$
    \State $P \gets \exp(S - m_i^{\text{new}})$
    \State $\alpha \gets e^{m_i - m_i^{\text{new}}}$ \Comment{online-softmax rescale factor}
    \State $\ell_i \gets \alpha \, \ell_i + \text{rowsum}(P)$
    \State $\mathbf{o}_i \gets \alpha \, \mathbf{o}_i + P \cdot V_n$
    \State $m_i \gets m_i^{\text{new}}$
\EndFor
\State $\mathbf{o}_i \gets \mathbf{o}_i / \ell_i$ \Comment{epilogue: normalize}
\State Write $\mathbf{o}_i$ to $o[m]$
\State Save $L = m_i + \log \ell_i$ to \texttt{softmax\_lse}[m] \Comment{used by backward}
\end{algorithmic}
\end{algorithm}

The backward main kernel (Algorithm~\ref{alg:dualkv-backward}) and the post-kernel context-gradient cast (Algorithm~\ref{alg:dualkv-bwd-convert}) are in Section~\ref{sec:kernel-backward}.

\section{Kernel Interface: Python Signature and Tensor Shapes}
\label{app:kernel-interface}

This appendix gives the full Python signature of the DualKV kernel (referenced from Section~\ref{sec:kernel}) and the per-tensor shape contract.

\begin{lstlisting}[language=Python, caption=Signature of DualKV kernel interface \texttt{flash\_attn\_dualkv\_varlen\_func}]
out = flash_attn_dualkv_varlen_func(
    q,                      # (sum_Ri, H,   d)   decoded queries, packed varlen
    k_context, v_context,   # (P,      H_k, d)   shared context KV (single copy)
    k_decoded, v_decoded,   # (sum_Ri, H_k, d)   per-sequence decoded KV, packed varlen
    cu_seqlens_q,           # (N+1,)             varlen offsets into q, k_decoded, v_decoded
    cu_seqlens_k_decoded,   # (N+1,)             same offsets (context is shared)
    max_seqlen_q,           # int                max R_i
    context_seqlen,         # int = P            causal offset: decoded query at index r
                            #                      has logical position P + r
    max_seqlen_k_decoded,   # int                max R_i
    softmax_scale=None,     # float              defaults to 1/sqrt(d)
    causal=True,            # bool               training always sets True
)
\end{lstlisting}

The key contract is that $k_{\text{ctx}}$ and $v_{\text{ctx}}$ (\texttt{k\_context}, \texttt{v\_context} in the listing) are stored once (not replicated $N$ times), while $q$, $k_{\text{dec}}$, and $v_{\text{dec}}$ are varlen-packed over the $N$ decoded sequences. The kernel is implemented as a PyTorch \texttt{torch.autograd.Function}. Internally, the forward also computes a per-query softmax log-sum-exp tensor \texttt{softmax\_lse} of shape $(H, \sum_i R_i)$ that is saved onto the autograd context for the backward pass (Section~\ref{sec:kernel-backward}); it is not part of the user-facing return value.


\section{Implementation Details: Extending the Kernel to hd512 and SWA}
\label{app:hdim-swa}

This appendix documents the two kernel variants added for Gemma-4 support (Section~\ref{sec:kernel}): a head-dimension-$512$ full-attention path and kernel-native causal sliding-window attention (SWA) for head dimension $256$. Both are compile-time specializations of the DualKV forward/backward kernels; the launcher dispatches on head dimension and window at runtime. FA2 supports only $d \le 256$, so the $d{=}512$ path is what enables Gemma-4's global layers under shared-prompt deduplication.

\textbf{Head dimension 512 (global layers).} The forward launches \texttt{Flash\_fwd\_kernel\_traits} with
head dim $512$, block sizes $B_M{=}B_N{=}64$, and $4$ warps, restricted to full causal attention (no windowing). Because
the $d{=}512$ tiles exceed the $48$\,KB static shared-memory limit, the kernel opts into dynamic shared memory via
\texttt{cudaFuncAttributeMaxDynamicSharedMemorySize} at launch. The backward is the tighter case: at $d{=}512$ both
shared memory and registers are under pressure, so it uses a smaller key block ($B_M{=}64$, $B_N{=}32$, $8$ warps),
holds $V$ in registers, and single-buffers the KV pipeline (no double buffering) to fit within the device
shared-memory-opt-in ceiling (queried at runtime via \texttt{cudaDevAttrMaxSharedMemoryPerBlockOptin}). These are the
standard FA2 levers for large head dims, applied within DualKV's two-region iteration. The $d{=}512$ path is used only
for Gemma-4's full-attention (global) layers.

\textbf{Sliding-window attention (local layers).} Gemma-4's local layers use head dim $256$ with a causal
sliding window $W$. We add kernel-native SWA to the DualKV kernel (compile-time \texttt{Is\_local}) for $d \le 256$:
within the two-region iteration, the window bound is applied in the logical position space (context positions
$0{\ldots}P{-}1$ followed by decoded positions $P{+}r$), so a decoded query at logical position $P{+}r$ attends only the
last $W$ logical keys spanning the prompt tail and its own response so far --- identical semantics to running windowed
FA2 over the contiguous $[K_c; K_d^{(i)}]$ sequence, but without replicating the shared context $N$ times. Out-of-window
context tiles are skipped at the block level. The result is verified mathematically equivalent to windowed FA2 (forward
and backward) by the kernel unit tests.

\textbf{Per-layer dispatch.} A single DualKV forward over Gemma-4 dispatches per decoder layer:
full-attention $d{=}512$ for the $10$ global layers, windowed $d{=}256$ for the $50$ sliding layers. The shared-prompt
repack (Section~\ref{sec:dualkv-packing}) is layer-agnostic, so the same packed $P{+}NR$ sequence feeds both layer types.

\section{Gemma-4-31B Evaluation with Sequence Parallelism}
\label{app:gemma4-sp-evaluation}

To show DualKV keeps pace with modern model architectures, we validate it end-to-end on Gemma-4-31B --- a
hybrid stack of interleaved head-dim-512 global and head-dim-256 sliding-window layers that standard FA2
cannot support --- exercising both the hd512/SWA kernel (Section~\ref{sec:kernel}, Appendix~\ref{app:hdim-swa})
and the SP composition (Section~\ref{sec:dualkv-sp-method}). Runs use LongReason (thinking enabled), $N{=}16$ rollouts, three
epochs, on a single 8$\times$H200 node; SP shards the actor and reference forward/backward.
Table~\ref{tab:dualkv-sp-gemma4} reports the configurations we ran. First-token log-probs
agree with the SP=1 computation to max$|\Delta|{=}0$ across SP${\in}\{1,2,4,8\}$, and held-out accuracy
remained stable with no degradation across all four configurations.

At 16K context, SP=4 lowers per-GPU peak memory from $96.8$ to $77.0$\,GB while adding ${\sim}12\%$ step
time --- the expected cost of the two per-layer all-to-alls, spread across the three DualKV-accelerated
training phases. At 32K context, where the deduplicated sequence is larger, SP=4 keeps peak memory flat at
$79.8$\,GB (well within the 140\,GB device) and completes full three-epoch training, demonstrating that
DualKV composes with SP to reach contexts beyond a single rank's capacity while retaining the shared-prompt
memory savings. At 64K with SP=8, peak memory stays flat at $78.9$\,GB --- comparable to the 16K and 32K runs ---
since SP shards the sequence so the per-rank workload stays bounded as context grows; larger SP degrees follow
by the same construction. DualKV's DAPO support (Section~\ref{sec:e2e-longreason}) and Gemma-4
support are orthogonal --- both leave the shared-prompt attention and packing path unchanged --- so they
compose without additional kernel work. The hd512/SWA kernels, the SP composition, and the Gemma-4 run
scripts are released on the \texttt{gemma4-dev} branch of the open-source
repository:\ \textcolor{blue}{\url{https://github.com/amazon-science/dualkv-flash-attn-for-rl/tree/gemma4-dev}}.

\begin{table}[ht]
\centering
\small
\setlength{\tabcolsep}{5pt}
\caption{DualKV with Ulysses SP on Gemma-4-31B (LongReason, $N{=}16$, thinking enabled, three
epochs; single 8$\times$H200 node). Wall-clock values are per-step means; peak memory is
\texttt{max\_memory\_allocated}. SP=1 and SP=4 at 16K are separate runs of identical config (no offload,
\texttt{gpu\_mem}${=}0.4$). Rollout generation (vLLM, SP-independent) is the remainder to the step total
and is omitted. Held-out accuracy was stable with no degradation in all rows.}
\label{tab:dualkv-sp-gemma4}
\begin{tabular}{@{}l c c c c c c@{}}
\toprule
\textbf{Config} & \textbf{Context} & \textbf{SP} & \textbf{Peak Mem (GB)}
  & \textbf{old\_log\_prob (s)} & \textbf{ref\_log\_prob (s)} & \textbf{Policy Update (s)} \\
\midrule
DualKV & 16K & 1 & 96.8
  & 140 & 140 & 549 \\
DualKV & 16K & 4 & 77.0
  & 157 & 178 & 638 \\
DualKV & 32K & 4 & 79.8
  & 315 & 319 & 1242 \\
DualKV & 64K & 8 & 78.9
  & 782 & 791 & 3073 \\
\bottomrule
\end{tabular}
\end{table}

\section{Composing DualKV with Sequence Parallelism}
\label{app:dualkv-sp-method}

DualKV and sequence parallelism (SP) are two orthogonal techniques for scaling long-context training:
DualKV deduplicates the shared prompt \emph{across} the $N$ sequences of a group, while Ulysses SP shards
\emph{one} packed sequence across ranks. Scaling long-context training commonly requires SP, which
adds step latency through per-layer all-to-all communication. In
GRPO/DAPO, where $N$ responses share a large $P$-token prompt, the dominant redundancy is prompt replication:
SP shards this replicated $N(P{+}R)$ sequence across ranks but never removes the redundancy, paying all-to-all
communication to distribute it, whereas DualKV eliminates the replication at its source ($N(P{+}R)\to P{+}NR$)
--- a more efficient approach for shared-prompt RL. SP remains necessary at ultra-long context, where
even the deduplicated $P{+}NR$ sequence exceeds one rank, so DualKV and SP must be used together. Making them
work together requires two adaptations.

\textbf{Repack before slice; all-to-all inside attention.} The DualKV repack (Section~\ref{sec:dualkv-packing})
runs \emph{before} the Ulysses sequence slice, so the shared prompt stays a single contiguous block; the
resulting $P{+}NR$ sequence is then padded to a multiple of the SP degree and sliced contiguously across ranks
like any sequence, and all packed side tensors (rolled labels, per-token temperature) are sliced identically
so they stay aligned. Attention is restored exactly as in standard Ulysses SP: an all-to-all before the DualKV
kernel gathers the full token sequence while scattering heads (each rank holds all $P{+}NR$ tokens but $H/\text{SP}$
heads), the two-region DualKV kernel runs unchanged on that per-rank shard, and a second all-to-all after attention
scatters the sequence back and gathers heads. Grouped-query attention is handled by repeating KV heads to a
multiple of the SP degree before the head scatter, exactly as in the FA2 SP path. The DualKV repack, kernel,
and per-group \texttt{cu\_seqlens} are unchanged under SP; the only SP-specific operations are the two
all-to-alls wrapping attention.

\textbf{Shared prompt logits under SP.} The other place DualKV interacts non-trivially with SP is
the fused log-probability head, which --- during the training forward pass --- computes each
already-generated token's log-prob for the policy loss without materializing full logits. DualKV's packing
places each prompt group's prompt once, so the first token
of \emph{every} response in the group is predicted by the logits at the last prompt position
$P{-}1$; the fused log-prob head computes these first-token log-probs from that shared row (the remaining tokens
use the standard shifted log-probs). Under SP the packed sequence is sliced contiguously, so this global
position $P{-}1$ resides on exactly \emph{one} rank, while the responses that need it are spread across
\emph{all} ranks. We resolve this with a lightweight collective. Let $p_g$ be the global index of group $g$'s
last prompt token in the packed sequence (Eq.~\ref{eq:dualkv-packing}), and
$\mathrm{parts}{=}\lceil(P{+}NR)/\mathrm{SP}\rceil$ the per-rank shard length; then rank $\mathrm{owner}(g){=}\lfloor
p_g/\mathrm{parts}\rfloor$ holds that token's hidden state at local index
$\mathrm{local}(g){=}p_g \bmod \mathrm{parts}$. Each rank
projects only the shared rows it owns and an all-reduce over the SP group reconstructs, for every group, the
exact SP=1 shared hidden state $h_{p_g}$:
\begin{equation}
\tilde{h}_g \;=\; \sum_{r=0}^{\mathrm{SP}-1} \mathds{1}[\mathrm{owner}(g){=}r]\; h^{(r)}_{\mathrm{local}(g)},
\qquad \ell_g \;=\; \tilde{h}_g\, W_{\mathrm{LM}}^\top / \tau,
\label{eq:dualkv-sp-gather}
\end{equation}
where $h^{(r)}$ is rank $r$'s local hidden shard and $\ell_g$ the shared prompt's logits at the last position (temperature
$\tau$) from which each response's first-token log-prob is read. Since each position has exactly one owner,
the indicator sum reconstructs $h_{p_g}$ exactly. The payload is $G$ hidden vectors for $G$ prompt groups, so
communication is $O(G\,d)$, independent of context length. A parity test confirms the first-token log-probs
agree with the SP=1 path to max$|\Delta|{=}0$ across SP${\in}\{1,2,4,8\}$
(Appendix~\ref{app:gemma4-sp-evaluation}).

\section{DualKV with FlashAttention-3}
\label{app:dualkv-fa3}

DualKV--FA3 preserves the shared-prompt design of Section~\ref{sec:kernel}. Its implementation uses
separate context and response attention calls, an exact log-sum-exp merge, and FP32 shared-gradient reduction.

We benchmark FA3 and DualKV--FA3 on one NVIDIA H100 SXM 80GB using FP16 and Qwen3-8B dimensions
($H{=}32$, $H_k{=}8$, $d{=}128$), with response length $R{=}2048$. Measurements include prompt
self-attention and all response attention, together with their backward computation; DualKV computes
the prompt once. Each of four independent runs uses five warmups and ten measurements, selecting the
upper median separately for each phase.

\begin{table}[ht]
\centering
\scriptsize
\caption{FA3 versus DualKV--FA3 wall-clock time and peak allocated memory reduction across rollout
factor $N$ and prompt length $P$; response length $R{=}2048$ throughout. All times are in milliseconds.
Combined time is forward plus backward. Times and speedups are averaged across runs; speedups are
not computed from the displayed means.}
\label{tab:benchmark-fa3-fwdbwd}
\begin{tabular}{@{}rr|rrr|rrr|rrr|r@{}}
\toprule
& & \multicolumn{3}{c|}{\textbf{FA3 (ms)}} & \multicolumn{3}{c|}{\textbf{DualKV--FA3 (ms)}} & \multicolumn{3}{c|}{\textbf{Speedup}} & \textbf{Mem} \\
$N$ & $P$ & fwd & bwd & f+b & fwd & bwd & f+b & fwd & bwd & f+b & $\downarrow$ \\
\midrule
28 & 4K  & 12.63  & 42.93   & 55.60   & 12.98  & 32.85  & 45.83  & 0.98$\times$ & 1.31$\times$ & 1.21$\times$ & 36\% \\
28 & 16K & 113.75 & 334.05  & 447.80  & 32.80  & 95.53  & 128.35 & 3.47$\times$ & 3.50$\times$ & 3.49$\times$ & 72\% \\
16 & 32K & 235.68 & 655.80  & 891.48  & 43.23  & 125.00 & 168.30 & 5.45$\times$ & 5.24$\times$ & 5.30$\times$ & 77\% \\
28 & 32K & 424.65 & 1134.13 & 1558.78 & 67.25  & 192.50 & 259.70 & 6.31$\times$ & 5.89$\times$ & 6.00$\times$ & 78\% \\
16 & 64K & 917.00 & 2387.50 & 3304.50 & 111.53 & 311.30 & 422.85 & 8.22$\times$ & 7.67$\times$ & \textbf{7.82$\times$} & \textbf{81\%} \\
\bottomrule
\end{tabular}
\end{table}

DualKV--FA3 achieves $1.21$--$7.82\times$ combined forward/backward speedup and reduces peak allocated
memory by $36$--$81\%$ relative to FA3. Benefits increase with prompt length and response count. At 4K,
the forward pass is slightly slower, while combined forward/backward remains faster. These results
demonstrate that shared-prompt reuse remains effective with FA3.

Output and gradient correctness are validated against a dense FP32 reference on reduced-size problems,
supplemented by large-shape forward comparisons, repeated backward executions, and CUDA Compute
Sanitizer checks. The reported results cover FP16 attention with equal-length responses; end-to-end
DualKV--FA3 training remains outside this evaluation.

\textbf{Open-source code.} The DualKV--FA3 implementation and benchmark scripts are available on the
\href{https://github.com/JiadingGai/dualkv-flash-attn-for-rl/tree/fa3_devel}{\texttt{fa3\_devel}}
branch.

\clearpage
\section{Additional Experimental Figures}
\label{app:additional-figures}

\begin{figure}[ht]
\centering
\includegraphics[width=\linewidth]{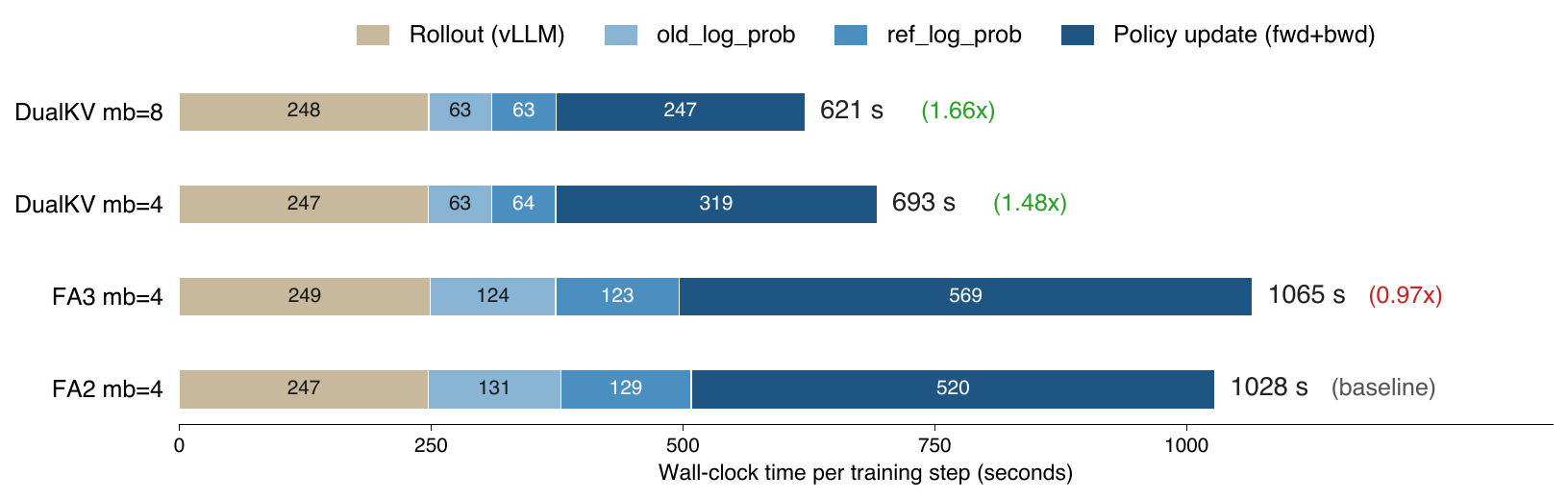}
\caption{Per-phase breakdown of an average GRPO training step on LongReason (Section~\ref{sec:e2e-longreason}). DualKV compresses the three training-side phases (\texttt{old\_log\_prob}, \texttt{ref\_log\_prob}, \texttt{policy update}) by ${\sim}2\times$ each; rollout latency is identical across all four configs. \textbf{(i)} the policy-update slice shrinks from $520$\,s (FA2) to $247$\,s (DualKV mb=8), a $2.1\times$ speedup; \textbf{(ii)} the rollout slice is invariant (${\sim}247$\,s across all four configs); \textbf{(iii)} rollout's share of step time grows from $23\%$ (FA2 mb=4) to $40\%$ (DualKV mb=8) --- DualKV shifts the bottleneck from the policy update to rollout generation.}
\label{fig:longreason-step-breakdown}
\end{figure}

\begin{figure}[ht]
\centering
\includegraphics[width=\linewidth]{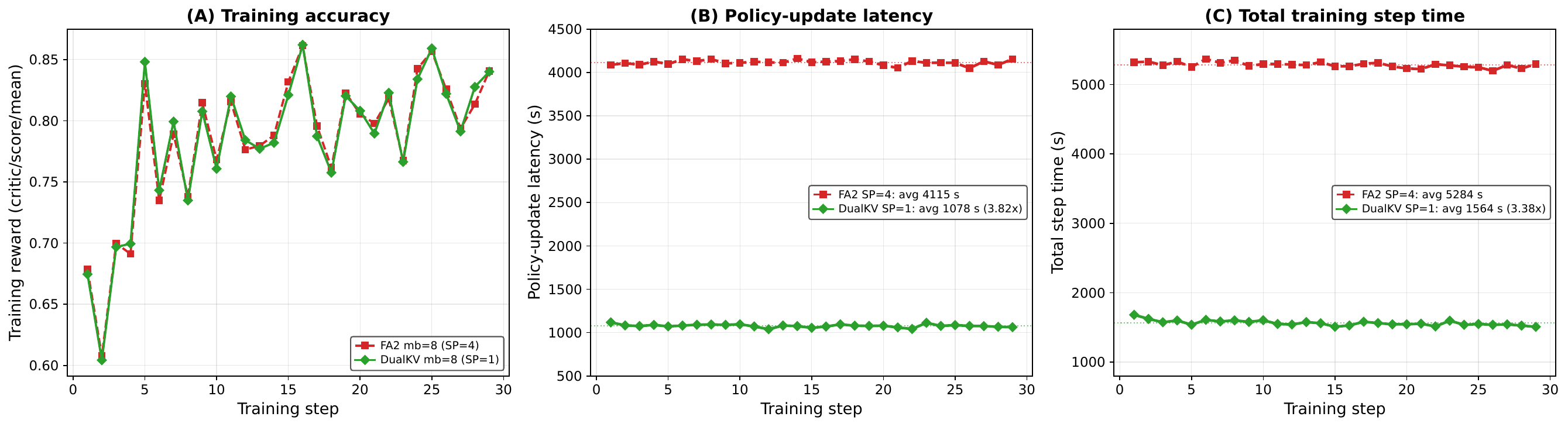}
\caption{Multi-node MoE GRPO training on LongReason (Section~\ref{sec:e2e-moe-multinode}; Qwen3-30B-A3B, 16$\times$H100, 2 nodes), comparing DualKV mb=8 SP=1 against FA2 mb=8 SP=4 over the common 1--29-step window. \textbf{(A)} Per-step training reward --- both configs track closely. \textbf{(B)} Per-step policy-update latency --- DualKV averages 1078\,s vs FA2's 4115\,s ($3.82\times$). \textbf{(C)} Per-step total training step time --- DualKV averages 1564\,s vs FA2's 5284\,s ($3.38\times$).}
\label{fig:moe-multinode}
\end{figure}

\begin{figure}[ht]
\centering
\includegraphics[width=\linewidth]{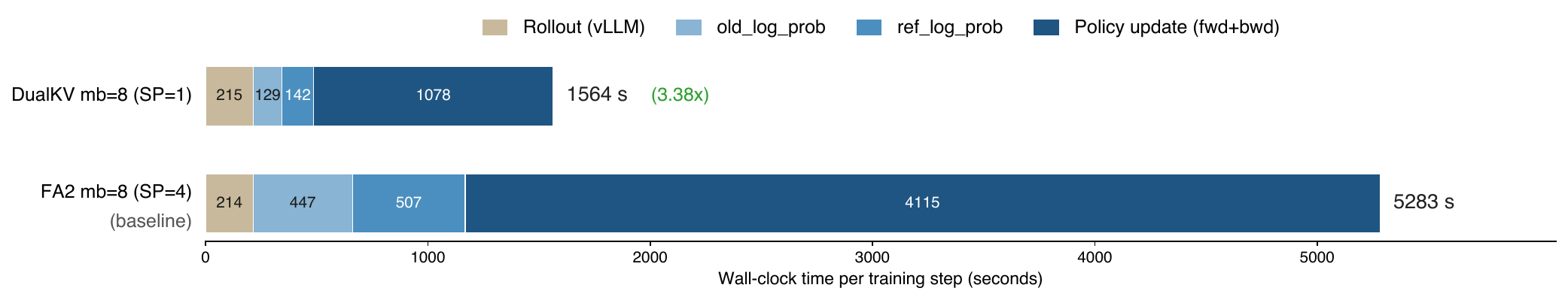}
\caption{Per-phase breakdown of an average GRPO training step for the multi-node MoE experiment (Section~\ref{sec:e2e-moe-multinode}). DualKV compresses the three training-side phases by $3.4$--$3.6\times$ each; rollout latency is identical across both configs. The FA2 bar uses 4-way Ulysses SP --- the smallest SP that avoids OOM at this scale --- while DualKV runs at SP=1.}
\label{fig:moe-multinode-breakdown}
\end{figure}

\end{document}